\documentclass[11pt, a4paper]{article}

\usepackage[utf8]{inputenc}
\usepackage[T1]{fontenc}
\usepackage{geometry}
\usepackage{amsmath}
\usepackage{amssymb}
\usepackage{amsthm}
\usepackage{amsfonts}
\usepackage{mathtools} % [修复] 必须加这个包才能用 \coloneqq (:= 符号)

\usepackage{enumitem}  % [修复] 必须加这个包才能用 [leftmargin=*]
\usepackage{booktabs}
\usepackage{multirow}
\usepackage{tabularx}

\usepackage{graphicx}
\usepackage[dvipsnames]{xcolor}

\usepackage[numbers,sort&compress]{natbib} % 推荐加上参数让引用更整齐

\usepackage{hyperref}
\hypersetup{
    colorlinks=true,
    linkcolor=blue,
    citecolor=blue,
    urlcolor=blue,
    pdftitle={Task-Driven Kernel Flows},
    pdfauthor={Anonymous}
}

\newtheorem{theorem}{Theorem}
\newtheorem{proposition}[theorem]{Proposition}
\newtheorem{lemma}[theorem]{Lemma}

\newtheorem{remark}[theorem]{Remark}

\DeclareMathOperator{\Tr}{Tr}       % [修复] 大写 Tr (日志里报错的是这个)
\DeclareMathOperator{\rank}{rank}   % [修复] 确保 rank 可用

\DeclareMathOperator{\Cov}{Cov}     % [修复] 新增 Cov 定义
\DeclareMathOperator*{\argmin}{argmin} % [修复] 新增 argmin 定义 (*号让下标居中)

\title{\textbf{Task-Driven Kernel Flows:\\ Label Rank Compression and Laplacian Spectral Filtering}}
\author{
  \begin{minipage}[t]{0.45\textwidth}
    \centering
    \textbf{Hongxi Li} \\
    School of Computer Science and Engineering \\
    Sun Yat-sen University \\
    \texttt{lihx228@mail2.sysu.edu.cn}
  \end{minipage}
  \hfill 
  \begin{minipage}[t]{0.45\textwidth}
    \centering
    \textbf{Chunlin Huang} \\
    Boya Junior High School \\
    Guangxi Hope High School \\
    \texttt{7wgn0326@gmail.com}
  \end{minipage}
}

\date{} 
\begin{document}
\maketitle

\begin{abstract}
  We develop a kernel-centric theory of task-driven feature learning in wide neural networks
  with linear readout and $\ell_2$-regularization. Our analysis proceeds in two stages to bridge the
  gap between interpretable dynamics and structural guarantees. First, operating in a fast-
  readout (adiabatic) regime with squared loss, we derive a closed-form kernel ODE governed
  by the competition between a task-dependent drive operator and isotropic regularization
  decay. This reveals the precise mechanism of alignment: for supervised learning with $C$
  outputs, the drive is a rank-$C$ operator that compresses the kernel into a low-dimensional
  subspace and obeys an explicit ``water-filling'' spectral law.

  Second, we show that the resulting structural phenomena are not artifacts of the time-scale
  separation. Under standard $\ell_2$-regularization and a $C$-dimensional linear readout, any stable
  steady state of the coupled feature--readout dynamics necessarily exhibits label-driven rank
  compression ($\mathrm{rank}(K_\infty) \le C$), and for squared loss satisfies the same spectral truncation law.
  These results are algebraic consequences of the architecture and loss, and do not depend on
  the fast-readout approximation.

  Complementing this deterministic picture, we analyze SGD noise at the kernel level and show
  that, for any convex loss with $C$ outputs, the instantaneous noise covariance is confined to a
  low-dimensional subspace of rank at most $O(C)$, independent of the network width and the
  current parameter values. Thus stochasticity induces restricted diffusion within the task-
  relevant subspace rather than isotropic exploration.

  We further extend our framework in two idealized directions: (i) a population limit, where
  we relate spectral evolution of the kernel integral operator to the bias--variance trade-off;
  and (ii) a stylized self-supervised kernel model driven by a graph Laplacian and a log-det
  repulsion, which produces high-rank, Laplacian-aligned representations. Together, these
  results provide a unified spectral language that contrasts the compressive nature of
  supervised learning with the expansive behavior of self-supervision, while clarifying which
  aspects are rigorous architectural consequences and which arise within specific kernel models.
\end{abstract}

\section{Introduction}

Deep learning owes its empirical success largely to its ability to learn data-dependent representations, or \emph{features}, that adapt to the underlying task structure. Classical learning theory and the Neural Tangent Kernel (NTK) regime \cite{jacot2018neural, arora2019exact} typically describe networks in the infinite-width limit where features remain static (the ``lazy training'' regime \cite{chizat2019lazy}). While theoretically convenient, this perspective fails to capture the rich feature learning dynamics observed in practice, where the kernel evolves significantly to align with the target function \cite{yang2021feature, atanasov2020neural}. Understanding the mechanism governing this kernel evolution is a central challenge in the theory of deep learning.

In this work, we propose a \emph{kernel-centric} framework to analyze feature learning in wide neural networks with a $C$-dimensional linear readout and explicit $\ell_2$-regularization. Instead of tracking high-dimensional parameter trajectories, we focus on the dynamics of the empirical kernel matrix $K(t) \in \mathbb{R}^{N \times N}$, which encodes the pairwise similarities of data representations. We adopt a dual-perspective approach to dissect the learning process:

\textbf{1. Dynamics via adiabatic approximation.}
To gain analytical insight into the \emph{trajectory} of learning, we first adopt a ``fast--slow'' regime, where the linear readout evolves significantly faster than the features. This allows us to ``integrate out'' the readout, yielding a closed-form Ordinary Differential Equation (ODE) for the kernel in an idealized, squared-loss setting. This ODE reveals the physical forces at play: a task-specific drive that promotes alignment and a regularization term that induces decay.

\textbf{2. Structural properties beyond time-scale separation.}
We then show that the key \emph{structural} conclusions suggested by this ODE are not artifacts of the fast--slow approximation. Under standard $\ell_2$-regularization and a $C$-dimensional linear readout, we prove that the \textbf{steady-state topology} (label-driven rank compression and, for squared loss, a spectral truncation law) and the \textbf{geometric structure of SGD noise} are enforced by the loss landscape and network architecture. Our findings on rank compression provide a theoretical grounding for the widely observed phenomenon of \emph{Neural Collapse} \cite{papyan2020prevalence, han2021neural}, where within-class variability vanishes at the end of training.

Our framework thus provides both a mechanistic description of \emph{how} features align over time (via the kernel ODE in a fast-readout regime) and rigorous guarantees on \emph{what} they can converge to at steady state (via fixed-point analysis that does not rely on time-scale separation). We further extend this analysis in two idealized directions: (i) to the population limit, to discuss generalization via the evolution of the kernel integral operator; and (ii) to a stylized Self-Supervised Learning (SSL) model, to highlight the spectral contrast between supervised compression and contrastive or redundancy-reduction based expansion \cite{chen2020simple, jing2022understanding}.

\textbf{Contributions.} Under these modeling assumptions, our main contributions are:

\begin{itemize}
    \item \textbf{Derivation of feature learning dynamics.}
    In a fast-readout (adiabatic) regime, we derive an explicit kernel ODE for wide networks with a linear readout and $\ell_2$-regularization, valid for any convex loss at the level of the driving term. It characterizes feature learning as a thermodynamic competition between task-alignment forces and regularization-induced decay.

    \item \textbf{Rank compression and spectral truncation at steady state.}
    We prove that for $C$-class supervised learning with a $C$-dimensional linear head and $\ell_2$-regularization, any stable steady-state kernel has rank at most $C$, regardless of the relative learning rates of the readout and features. For squared loss, we derive an explicit spectral truncation (``water-filling'') law consistent with the spectral bias observed in deep networks \cite{rahaman2019spectral}. Importantly, these steady-state properties hold for general coupled dynamics, beyond the time-scale separation used to derive the ODE.

    \item \textbf{Intrinsic low-rank SGD noise.}
    We analyze the stochastic gradients at the kernel level and show that, for any convex loss with $C$ outputs, the instantaneous noise covariance matrix is confined to a low-rank subspace determined by the output dimension $C$. This architectural constraint acts as a built-in filter, forcing SGD noise to lie in (and thus diffuse within) the task-relevant subspace rather than exciting arbitrary kernel directions.

    \item \textbf{Population limit and generalization.}
    We extend our framework to the infinite-sample limit, defining the evolution of the associated kernel integral operator and showing how the spectral truncation mechanism directly shapes the bias–variance trade-off on unseen data, in contrast to fixed-kernel regimes such as NTK.

    \item \textbf{Spectral unification of supervised and self-supervised learning.}
    We analyze a stylized kernel model for SSL driven by a graph Laplacian \cite{belkin2003laplacian} and a log-determinant repulsion. In this model, the learned kernel has a high-rank, Laplacian-aligned spectrum with an explicit $(\nu_i + \mathrm{const})^{-1}$ shape, leading to ``whitened'' representations \cite{zbontar2021barlow}. This provides a unified spectral language to contrast the compressive, low-rank nature of supervision with the expansive, high-rank nature of self-supervision.
\end{itemize}

\paragraph{Assumptions and scope.}
Throughout the paper we work with a standard but idealized setting:

(i) \textbf{Linear readout with $\ell_2$-regularization.}
The network output is produced by a $C$-dimensional linear head on top of the features.
We apply explicit $\ell_2$-regularization to both the readout and (in the free-feature model)
the backbone features.

(ii) \textbf{Feature-learning / wide-backbone regime.}
We assume the backbone is sufficiently expressive that the representation dynamics can be
modeled directly at the level of the empirical feature matrix $\Phi$ and its kernel
$K = \Phi^\top \Phi$, without further architectural constraints; the NTK / lazy regime is
\emph{not} our focus.

(iii) \textbf{Squared loss for closed-form spectral laws.}
Our closed-form kernel ODE and water-filling spectral truncation law are derived for the
squared loss, which yields an autonomous dynamics in the eigenbasis of the label Gram
matrix. For more general convex losses (e.g.\ cross-entropy), the same rank-compression
mechanism applies at steady state, but we do not claim closed-form spectral trajectories.

(iv) \textbf{Stylized SSL objective.}
For self-supervised and semi-supervised settings we study a stylized kernel objective based
on a graph Laplacian and a log-det repulsion. This is intended as a canonical spectral
model that makes the contrast between supervised compression and SSL expansion analytically
transparent, rather than an exact derivation of any particular method such as SimCLR or
BYOL.

Within this setting, our \emph{rigorous structural results}---such as label-driven rank
compression $\mathrm{rank}(K_\infty) \le C$, the equivalence between weight decay and a
nuclear-norm bias on the end-to-end mapping, and the low-rank structure of SGD noise---are
algebraic consequences of the $C$-dimensional output bottleneck and $\ell_2$-regularization.
By contrast, the \emph{exact} water-filling spectrum in the supervised case and the
$(\nu_i + \mathrm{const})^{-1}$ spectral shape in our SSL model should be viewed as
behaviors of these idealized kernel flows, not literal descriptions of all practical
training setups with cross-entropy, batch normalization, or attention.

\section{Model and Two-Time-Scale Dynamics} \label{sec:model}

\subsection{Intuition: The Fast-Readout Hypothesis}
\label{sec:intuition}

Deep neural networks can be structurally decomposed into two parts: a non-linear feature extractor $\varphi(\cdot)$, which maps inputs to a high-dimensional embedding space, and a linear readout head $W$, which maps embeddings to predictions. The dynamics of these two components are often fundamentally different.

To build intuition, consider a simplified scenario where the feature extractor is frozen (i.e., $\Phi$ is constant). In this case, optimizing the network reduces to training a linear model (e.g., linear regression or logistic regression) on fixed features. Since the loss function is typically convex with respect to $W$ (and strictly convex with $\ell_2$ regularization), the gradient dynamics for $W$ are simple: $W$ converges exponentially fast to a unique optimum, denoted as $W^*(\Phi)$.

In reality, of course, $\Phi$ evolves alongside $W$. However, as training progresses, we often observe that deep representations stabilize much slower than the top linear layer. This separation is even more pronounced in regimes such as transfer learning or when specific learning rate schedules (large $\eta_W$, small $\eta_\Phi$) are employed.

Based on this observation, we proceed with a \emph{time-scale separation} ansatz. We assume that the readout dynamics are sufficiently fast relative to the feature dynamics such that $W$ effectively equilibrates instantaneously.
$$ W(t) \approx W^*(\Phi(t)) \quad \text{for all } t. $$
This ``adiabatic'' approximation allows us to eliminate $W$ from the equations of motion. Instead of tracking the coupled system $(\Phi, W)$, we can focus entirely on the effective dynamics of the features driven by the \emph{optimal} readout. While this is a bold simplification, it captures the essential feedback loop: features evolve to minimize the loss, assuming the classifier will always make the best use of them.

Mathematically, this reduces the original objective $\mathcal{L}(\Phi, W)$ to an effective functional $\widetilde{\mathcal{L}}(\Phi)$ depending solely on the kernel, which we formalize next.

\subsection{Mathematical Formulation}
\label{sec:math_setup}

We consider a supervised learning task with a dataset of $N$ samples $X = [x_1, \dots, x_N] \in \mathbb{R}^{d_{in} \times N}$ and corresponding targets $Y = [y_1, \dots, y_N]^\top \in \mathbb{R}^{N \times C}$, where $C$ is the number of classes (or output dimensions).

The neural network is modeled as a composition of a feature map $\Phi(\cdot)$ and a linear readout $W \in \mathbb{R}^{C \times k}$. Let $\Phi \in \mathbb{R}^{k \times N}$ denote the collective feature matrix where the $i$-th column is $\phi_i = \Phi(x_i)$. The network output is given by $\hat{Y} = (W\Phi)^\top = \Phi^\top W^\top \in \mathbb{R}^{N \times C}$.
We study the training dynamics under a regularized empirical risk minimization framework. The total objective function $\mathcal{J}(W, \Phi)$ is defined as:
\begin{equation}
    \label{eq:total_loss}
    \mathcal{J}(W, \Phi) = \mathcal{L}(\hat{Y}, Y) + \frac{\lambda}{2} \|W\|_F^2 + \frac{\mu}{2} \|\Phi\|_F^2,
\end{equation}
where $\mathcal{L}$ is a convex loss function (e.g., squared error or cross-entropy), $\lambda > 0$ is the regularization coefficient for the readout, and $\mu \ge 0$ represents the weight decay for the feature extractor.

Our primary object of study is the \textbf{empirical kernel matrix} (or Gram matrix) $K \in \mathbb{R}^{N \times N}$, defined as the inner product of features:
\begin{equation}
    K(t) = \Phi(t)^\top \Phi(t).
\end{equation}
The kernel $K$ captures the geometry of the data representation. Importantly, while the parameter space of $\Phi$ may be vast (and potentially infinite in the wide limit), the dynamics of learning on a finite dataset are entirely encapsulated by the evolution of this $N \times N$ matrix.

\section{Derivation of the Kernel ODE}
\label{sec:derivation}

In this section, we derive the exact differential equation governing the evolution of $K(t)$ under the fast-readout assumption.

\subsection{The Fast-Readout Limit}
Following the intuition in Section \ref{sec:intuition}, we assume the readout $W$ evolves on a sufficiently fast time scale such that it effectively minimizes the objective $\mathcal{J}$ for the current fixed features $\Phi$ at every instant $t$. We define the optimal readout $W^*(\Phi)$ as:
\begin{equation}
    W^*(\Phi) = \argmin_{W} \left( \mathcal{L}((W\Phi)^\top, Y) + \frac{\lambda}{2} \|W\|_F^2 \right).
\end{equation}
Substituting $W^*$ back into Eq. \eqref{eq:total_loss} yields the \emph{effective feature loss} $\widetilde{\mathcal{L}}(\Phi) = \mathcal{J}(W^*(\Phi), \Phi)$. Feature learning is then modeled as a gradient flow on this effective landscape:
\begin{equation}
    \label{eq:phi_flow}
    \dot{\Phi} = - \nabla_\Phi \widetilde{\mathcal{L}}(\Phi).
\end{equation}

\subsection{General Kernel Dynamics}
We first derive a general evolution equation valid for any convex loss function $\mathcal{L}$.
Applying the envelope theorem, the total derivative of the effective loss with respect to $\Phi$ is simply the partial derivative of the joint loss evaluated at the optimum $W^*$:
\begin{equation}
    \nabla_\Phi \widetilde{\mathcal{L}}(\Phi) = \nabla_\Phi \mathcal{J}(W, \Phi) \Big|_{W=W^*}.
\end{equation}
Using the chain rule on Eq. \eqref{eq:total_loss}, we obtain:
\begin{equation}
    \nabla_\Phi \mathcal{J} = W^\top \frac{\partial \mathcal{L}}{\partial \hat{Y}^\top} + \mu \Phi = -W^\top R^\top + \mu \Phi,
\end{equation}
where we define the \textbf{generalized residual matrix} $R \in \mathbb{R}^{N \times C}$ as the negative gradient of the loss with respect to predictions: $R \coloneqq -\nabla_{\hat{Y}} \mathcal{L}$. For squared loss, $R = Y - \hat{Y}$.
Substituting this into Eq. \eqref{eq:phi_flow}, the feature dynamics become:
\begin{equation}
    \label{eq:phi_dynamics}
    \dot{\Phi} = W^{*\top} R^\top - \mu \Phi.
\end{equation}
This equation reveals that features evolve via two forces: a \emph{driving force} that pulls features to align with the back-propagated residual signal ($W^{*\top} R^\top$), and a \emph{decay force} ($-\mu \Phi$) induced by regularization.

Now, we compute the time derivative of the kernel $K = \Phi^\top \Phi$:
\begin{equation}
    \dot{K} = \dot{\Phi}^\top \Phi + \Phi^\top \dot{\Phi}.
\end{equation}
Plugging in Eq. \eqref{eq:phi_dynamics} and noting that $\hat{Y} = \Phi^\top W^{*\top}$:
\begin{align}
    \dot{K} &= (R W^* \Phi - \mu \Phi^\top \Phi) + (\Phi^\top W^{*\top} R^\top - \mu \Phi^\top \Phi) \nonumber \\
            &= R \hat{Y}^\top + \hat{Y} R^\top - 2\mu K.
    \label{eq:general_ode}
\end{align}
Eq.~\eqref{eq:general_ode} is the \textbf{task-driven kernel ODE}. It states that the kernel's rate of change is determined by the alignment between the model's predictions $\hat{Y}$ and the task residuals $R$, opposed by a uniform decay.

\paragraph{Interpretation: Hebbian-like Feedback.} 
Equation \eqref{eq:general_ode} admits a compelling physical interpretation. The driving term $R \hat{Y}^\top$ is the outer product between the residual error $R$ and the current prediction $\hat{Y}$. This is analogous to a supervised form of \emph{Hebbian learning} (``fire together, wire together''): the kernel strength increases along directions where the model's predictions actively correlate with the error signal. In contrast, the $-2\mu K$ term acts as a uniform forgetting mechanism. Feature learning thus emerges as a selection process: the network reinforces directions useful for reducing error while decaying irrelevant components.

\subsection{Closed-Form Dynamics for Squared Loss}
\label{sec:kernel_riccati}

To perform spectral analysis, we specialize to the case of Mean Squared Error (MSE), $\mathcal{L}(\hat{Y}, Y) = \frac{1}{2} \|\hat{Y} - Y\|_F^2$.

\textbf{1. Explicit Readout.} For MSE, the optimal readout $W^*$ is the solution to a ridge regression problem. Using the matrix inversion lemma, the prediction $\hat{Y}$ can be written in a kernelized form independent of the feature dimension $k$:
\begin{equation}
    \hat{Y} = K (K + \lambda I)^{-1} Y.
\end{equation}
\textbf{2. Explicit Residual.} Consequently, the residual $R = Y - \hat{Y}$ becomes:
\begin{equation}
    R = Y - K (K + \lambda I)^{-1} Y = \lambda (K + \lambda I)^{-1} Y.
\end{equation}
\textbf{3. The Drive Operator.}
Substituting $\hat{Y}$ and $R$ into the general ODE (Eq. \eqref{eq:general_ode}), the driving term $R \hat{Y}^\top + \hat{Y} R^\top$ becomes:
\begin{equation}
    \mathcal{D}(K) = \lambda (K+\lambda I)^{-1} Y Y^\top (K+\lambda I)^{-1} K + \text{h.c.},
\end{equation}
where h.c. denotes the Hermitian conjugate (transpose) of the first term. Since $K$ and $(K+\lambda I)^{-1}$ commute, we can rearrange terms. Let $M_Y = Y Y^\top$ be the \emph{label kernel matrix}. We arrive at the final closed-form ODE:
\begin{equation}
    \label{eq:mse_ode}
    \dot{K}(t) = \lambda \left[ (K+\lambda I)^{-1} M_Y (K+\lambda I)^{-1} K + K (K+\lambda I)^{-1} M_Y (K+\lambda I)^{-1} \right] - 2\mu K.
\end{equation}
This equation is the foundation of our subsequent analysis. The driving term depends explicitly on the label structure $M_Y$, which has rank at most $C$. This rank bottleneck is the origin of the compression phenomenon we discuss in Section 4.

\subsection{Intuition: Scalar Dynamics and Spectral Filtering}
\label{sec:scalar_intuition}

To demystify the matrix ODE in Eq. \eqref{eq:mse_ode}, consider a simplified scalar case where the data consists of a single sample with label $y$ and kernel value $k(t) \in \mathbb{R}$. The equation simplifies to:
\begin{equation}
    \label{eq:scalar_ode}
    \dot{k} = \frac{2\lambda k}{(\lambda + k)^2} y^2 - 2\mu k.
\end{equation}
This scalar dynamics highlights a crucial \textbf{signal-to-noise filter mechanism}:
\begin{itemize}
    \item \textbf{Reinforcement:} The growth term $\frac{k}{(\lambda+k)^2}$ is non-monotonic. It vanishes when $k \to 0$ (no features) and $k \to \infty$ (saturation), peaking when $k \approx \lambda$. This implies that the network actively reinforces features that are ``just right''—neither too weak to be useful nor too strong to be unstable.
    \item \textbf{Thresholding:} Feature growth is only possible if the signal strength (proportional to $y^2$) overcomes the regularization barrier $\mu$. If the label signal is too weak ($y^2 \lesssim \mu \lambda$), the decay term dominates, and the feature $k$ collapses to zero.
\end{itemize}

% --- 插入图一的代码开始 ---
\begin{figure}[t!]
    \centering
    % 请将 'figure1.png' 替换为你第一张图的实际文件名
    \includegraphics[width=\linewidth]{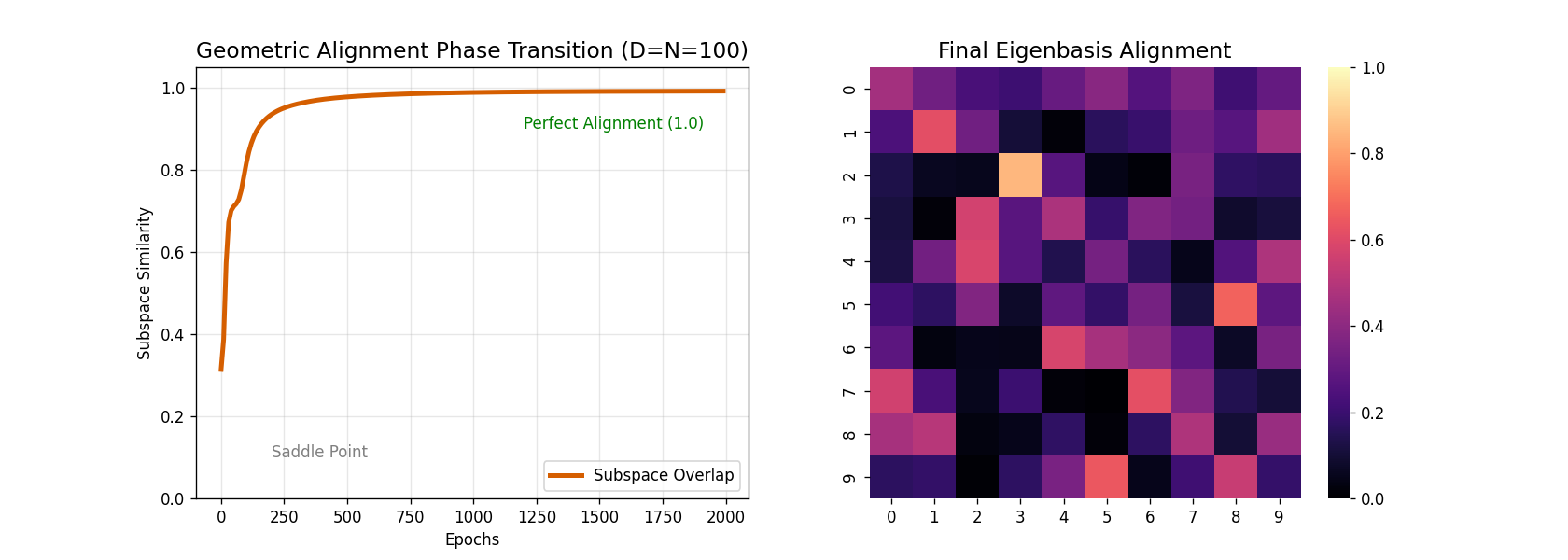} 
    \caption{\textbf{Dynamics of Geometric Alignment.} (Left) The evolution of the subspace overlap score during training (with $D=N=100$). The system exhibits a distinct phase transition, escaping a saddle point to achieve perfect alignment (Score $\approx 1.0$) with the target subspace. (Right) The heatmap of the final feature kernel $K_\infty$ entries. The checkerboard pattern confirms the commutativity structure $[K_\infty, M_Y] \approx 0$, validating that the learned features share the same eigenbasis as the labels.}
    \label{fig:geometric_alignment}
\end{figure}

As we show in the next section, this scalar intuition generalizes to the spectral domain: the matrix ODE applies a similar filter to the \emph{eigenvalues} of the kernel, selectively amplifying eigenmodes that align with the labels while truncating others (as visualized in Figure \ref{fig:geometric_alignment}).

\section{Convergence Analysis}
\label{sec:convergence}

Before characterizing the structural properties of the learned features, we must first establish that the learning dynamics are well-behaved mathematically. Although the matrix differential equation for $K(t)$ (Eq. \eqref{eq:mse_ode}) is non-linear and high-dimensional, its convergence properties follow directly from the construction of the effective loss function defined in Section \ref{sec:model}. Specifically, we prove that under the proposed dynamics, the kernel matrix $K(t)$ converges globally to a unique stationary state.

\subsection{Energy Landscape and Lyapunov Stability}

The stability analysis relies on identifying the effective loss $\widetilde{\mathcal{L}}(\Phi)$ as a Lyapunov function for the system. Since the dynamics are defined as a gradient flow, the system naturally seeks to minimize this energy function.

\begin{theorem}[Global Convergence]
\label{thm:convergence}
Assume the regularization coefficients satisfy $\lambda > 0$ and $\mu > 0$. For any initialization $\Phi(0)$, the feature trajectory $\Phi(t)$ remains bounded for all $t \ge 0$. Furthermore, the trajectory converges to a unique limit point $\Phi_\infty$, and consequently, the kernel matrix converges to a unique steady-state matrix $K_\infty$.
\end{theorem}

\begin{proof}
The proof proceeds in three steps: establishing monotonicity, proving boundedness, and invoking analytic convergence properties.

\textbf{1. Monotonicity.} By definition, the dynamics follow the gradient flow $\dot{\Phi} = - \nabla_\Phi \widetilde{\mathcal{L}}(\Phi)$. The time derivative of the effective loss along the trajectory is:
\begin{equation}
    \frac{d}{dt} \widetilde{\mathcal{L}}(\Phi(t)) = \left\langle \nabla_\Phi \widetilde{\mathcal{L}}, \dot{\Phi} \right\rangle = - \| \nabla_\Phi \widetilde{\mathcal{L}} \|_F^2 = - \| \dot{\Phi} \|_F^2 \le 0.
\end{equation}
Thus, the objective function is strictly non-increasing along the trajectory unless the system is at a critical point where $\dot{\Phi} = 0$.

\textbf{2. Boundedness (Coercivity).} The effective loss function consists of a non-negative data-fitting term (the minimum of the convex ridge regression problem) plus a regularization term on the features:
\begin{equation}
    \widetilde{\mathcal{L}}(\Phi) = \min_W \mathcal{J}(W, \Phi) \ge \frac{\mu}{2} \|\Phi\|_F^2 = \frac{\mu}{2} \Tr(K).
\end{equation}
Since the loss is non-increasing, we have $\widetilde{\mathcal{L}}(\Phi(t)) \le \widetilde{\mathcal{L}}(\Phi(0)) \coloneqq \mathcal{L}_0$ for all $t > 0$. This implies a bound on the Frobenius norm of the features:
\begin{equation}
    \|\Phi(t)\|_F^2 \le \frac{2 \mathcal{L}_0}{\mu}.
\end{equation}
Because the sublevel sets of the objective function are compact (guaranteed by the coercivity condition $\mu > 0$), the trajectory $\Phi(t)$ is confined to a bounded region in the parameter space $\mathbb{R}^{k \times N}$.

\textbf{3. Convergence.} While standard results such as LaSalle's Invariance Principle guarantee convergence to the \textit{set} of critical points, we can make a stronger statement. Since the objective function $\widetilde{\mathcal{L}}(\Phi)$ is real-analytic (it involves rational functions and polynomials of the matrix entries), the \textbf{\L ojasiewicz-Simon gradient inequality} applies. This inequality ensures that the trajectory has finite length and converges to a single unique critical point $\Phi_\infty$, rather than oscillating or drifting within a continuum of critical points. As a result, the limit $K_\infty = \Phi_\infty^\top \Phi_\infty$ is well-defined and unique.
\end{proof}

\subsection{The Fixed-Point Equation}

Having established that a limit exists, we now derive the condition that the steady-state kernel must satisfy. By setting the time derivative $\dot{K} = 0$ in Eq. \eqref{eq:mse_ode}, we obtain the algebraic fixed-point equation:

\begin{equation}
    \label{eq:fixed_point_algebraic}
     \lambda \left[ \Sigma_\infty M_Y \Sigma_\infty K_\infty + K_\infty \Sigma_\infty M_Y \Sigma_\infty \right] = 2\mu K_\infty,
\end{equation}
where we have defined the equilibrium \emph{resolvent matrix} as $\Sigma_\infty \coloneqq (K_\infty + \lambda I)^{-1}$ and the label correlation matrix as $M_Y \coloneqq Y Y^\top$.

This equation encapsulates the fundamental trade-off of the learning dynamics:
\begin{enumerate}
    \item The \textbf{Driving Force} (LHS): The term involving $M_Y$ represents the pressure from the labels to align the kernel with the target data structure.
    \item The \textbf{Regularization Force} (RHS): The term $2\mu K_\infty$ represents the penalty on feature complexity, which suppresses the eigenvalues of the kernel.
\end{enumerate}
In the next section, we will analyze the spectral properties of the solution $K_\infty$ to reveal how this balance leads to the phenomenon of rank compression.

\section{Spectral Analysis and Low-Rank Compression}
\label{sec:spectral_analysis}

In this section, we analyze the structural properties of the learned representation. We proceed in two steps:
\begin{enumerate}
    \item \textbf{Geometric Orientation (Structural):} For any convex loss within our $\ell_2$-regularized, $C$-output setting, we prove that the representation is compressed into a low-dimensional subspace determined by the labels.
    \item \textbf{Spectral Magnitude (Squared Loss):} For the squared loss, we derive the exact eigenvalues of the steady-state kernel, revealing a sharp phase transition (spectral truncation).
\end{enumerate}

\subsection{Label-Driven Rank Compression as an Architectural Law}
\label{subsec:rank_compression}

First, we characterize the "destination" of feature learning. Consider the steady-state equationderived from the general ODE (Eq. \ref{eq:general_ode}) by setting $\dot{K}=0$:
\begin{equation}
    \label{eq:steady_state_lyapunov}
    K_\infty M(K_\infty) + M(K_\infty) K_\infty = 2\lambda \mu K_\infty.
\end{equation}
Here, the driving matrix is $M(K) = B(K)B(K)^\top$. The crucial observation is dimensional: while $K_\infty \in \mathbb{R}^{N \times N}$, the residual matrix $B(K) \in \mathbb{R}^{N \times C}$ has only $C$ columns. Thus, the rank of the driving force is intrinsically bounded:
\begin{equation}
    \text{rank}(M(K_\infty)) = \text{rank}(B(K_\infty)) \le C.
\end{equation}

The following theorem proves that weight decay acts as a "dimensional guillotine," eliminating all feature dimensions not actively supported by this low-rank driving force.

\begin{theorem}[Label-Driven Rank Compression]

\label{thm:rank_compression}
For any regularization strength $\mu > 0$, the nullspace of the driving force $M(K_\infty)$ is contained in the nullspace of the learned kernel $K_\infty$. Consequently, the rank of the representation is bounded by the number of classes:
\begin{equation}
    \text{rank}(K_\infty) \le \text{rank}(M(K_\infty)) \le C.
\end{equation}
\end{theorem}

\begin{proof}
Let $v \in \mathbb{R}^N$ be any vector in the nullspace of the driving matrix, i.e., $M(K_\infty) v = 0$. Since $M(K_\infty)$ is symmetric, $v$ is also orthogonal to the image of $M(K_\infty)$.
Right-multiplying the steady-state equation \eqref{eq:steady_state_lyapunov} by $v$, we obtain:
\begin{equation}
    K_\infty \underbrace{M(K_\infty) v}_{0} + M(K_\infty) K_\infty v = 2\lambda \mu K_\infty v \implies M(K_\infty) K_\infty v = 2\lambda \mu K_\infty v.
\end{equation}
Now, we left-multiply by $v^\top$:
\begin{equation}
    v^\top M(K_\infty) K_\infty v = 2\lambda \mu v^\top K_\infty v.
\end{equation}
Observe the Left Hand Side (LHS): since $M(K_\infty)$ is symmetric, $v^\top M(K_\infty) = (M(K_\infty)v)^\top = 0$. Thus, the LHS is strictly zero.
The equation reduces to:
\begin{equation}
    0 = 2\lambda \mu (v^\top K_\infty v).
\end{equation}
Since $\lambda > 0$ and $\mu > 0$, we must have $v^\top K_\infty v = 0$. Because the kernel matrix $K_\infty$ is Positive Semi-Definite (PSD), $v^\top K_\infty v = 0$ implies $K_\infty v = 0$.

\textbf{Conclusion:} We have shown that $M(K_\infty) v = 0 \implies K_\infty v = 0$. In set-theoretic terms, $\text{ker}(M(K_\infty)) \subseteq \text{ker}(K_\infty)$. Taking the orthogonal complement implies $\text{Im}(K_\infty) \subseteq \text{Im}(M(K_\infty))$. Therefore, $\text{rank}(K_\infty) \le \text{rank}(M(K_\infty)) \le C$.
\end{proof}

\paragraph{Physical Interpretation.} This result provides a rigorous justification for the Neural Collapse phenomenon. It asserts that weight decay forces the network to "forget" any variation in the data that is not correlated with the label residuals. The feature space collapses from dimension $N$ (number of samples) down to $C$ (number of classes), regardless of the network width.

\subsection{Exact Solution: The Squared Loss Case}
\label{subsec:squared_loss}

While the rank compression theorem (Theorem \ref{thm:convergence}) establishes the existence of a low-rank limit, it provides an upper bound rather than an explicit characterization. To determine the precise magnitude of the learned features and the exact threshold for collapse, we specialize our analysis to the \textbf{Squared Loss}. 

In this setting, the residual map becomes linear, allowing us to solve the fixed-point equation analytically. This yields a closed-form law for the spectrum of the learned kernel, revealing a sharp phase transition between "signal" and "noise."

\subsubsection{The Alignment Principle: Geometry from Energy Minimization}
\label{subsec:alignment_principle}

We begin by determining the geometric relationship between the learned kernel $K_\infty$ and the task structure $M_Y$. While the algebraic fixed-point equation (Eq. \ref{eq:fixed_point_algebraic}) admits a solution, we must verify that this solution represents a stable, energy-minimizing configuration.

A fundamental question is: \textit{Why should the internal features align with the external labels?} The answer lies in the variational structure of the problem.

\begin{lemma}[Variational Alignment Principle]
\label{lem:alignment}
Let $\mathcal{L}(K)$ be the effective potential (loss) of the system. The global minimizers of $\mathcal{L}(K)$, and thus the stable steady states of the kernel dynamics, are configurations where the feature kernel $K_\infty$ and the label correlation matrix $M_Y$ are \textbf{simultaneously diagonalizable} (commute). Furthermore, their eigenvectors are aligned.
\end{lemma}

\begin{proof}
Recall the effective objective function derived in the adiabatic limit (Section \ref{sec:derivation}):
\begin{equation}
    \min_{K \succeq 0} \mathcal{J}(K) = \Tr\left( Y^\top (I + \lambda^{-1} K)^{-1} Y \right) + \mu \Tr(K).
\end{equation}
Using the cyclic property of the trace and defining $M_Y = YY^\top$, the data-fidelity term becomes:
\begin{equation}
    \mathcal{L}_{data} = \Tr\left( (I + \lambda^{-1} K)^{-1} M_Y \right).
\end{equation}
We invoke \textbf{von Neumann's Trace Inequality}, which states that for any two symmetric positive semi-definite matrices $A$ and $B$:
\begin{equation}
    \Tr(AB) \ge \sum_{i=1}^N \lambda_i(A) \lambda_{N-i+1}(B),
\end{equation}
where eigenvalues are sorted in descending order $\lambda_1 \ge \dots \ge \lambda_N$. The equality (minimum value) is achieved if and only if $A$ and $B$ share the same eigenvectors and the eigenvalues are paired in reverse order.
    
In our case, let $A = (I + \lambda^{-1} K)^{-1}$. The function $f(x) = (1 + x/\lambda)^{-1}$ is strictly decreasing. Therefore, to minimize $\Tr(A M_Y)$, the eigenvectors of $K$ must align with the eigenvectors of $M_Y$, and the largest eigenvalues of $K$ (which produce the smallest eigenvalues of $A$) must align with the largest eigenvalues of $M_Y$.
    
Any misalignment introduces an "off-diagonal" potential energy cost. Since the gradient flow dynamics naturally descend this potential landscape, the system is asymptotically driven to this commutative configuration:
\begin{equation}
    [K_\infty, M_Y] = 0.
\end{equation}
\end{proof}

\subsubsection{The Spectral Truncation Theorem}

By exploiting the simultaneous diagonalizability, we can project the matrix dynamics onto the eigenbasis of the task. Let $\{(\sigma_i, \mathbf{u}_i)\}_{i=1}^N$ be the eigenpairs of the data correlation matrix $M_Y$, where $\sigma_i$ represents the strength of the $i$-th task component (e.g., the variance of the data along a principal direction). Let $k_i$ denote the corresponding eigenvalue of the learned kernel $K_\infty$.

Projecting Eq. \eqref{eq:fixed_point_algebraic} onto eigenvector $\mathbf{u}_i$ yields the scalar balance equation:
\begin{equation}
    \lambda \left[ \frac{1}{(k_i + \lambda)} \sigma_i \frac{1}{(k_i + \lambda)} k_i + k_i \frac{1}{(k_i + \lambda)} \sigma_i \frac{1}{(k_i + \lambda)} \right] = 2\mu k_i.
\end{equation}
Simplifying the terms (assuming $k_i > 0$ to divide by $k_i$, or checking the $k_i=0$ case separately):
\begin{equation}
    \frac{2\lambda \sigma_i}{(k_i + \lambda)^2} = 2\mu \implies (k_i + \lambda)^2 = \frac{\lambda \sigma_i}{\mu}.
\end{equation}
Solving for $k_i$ and enforcing the non-negativity constraint ($K_\infty \succeq 0$) leads to our main spectral result:

\begin{theorem}[Spectral Truncation Law]
\label{thm:spectral_truncation}
Let $\tau \coloneqq \lambda \mu$ be the effective spectral noise threshold. The eigenvalues of the learned feature kernel under squared loss satisfy:
\begin{equation}
    \label{eq:water_filling}
    k_i = \lambda \left( \sqrt{\frac{\sigma_i}{\tau}} - 1 \right)_+,
\end{equation}
where $(x)_+ = \max(0, x)$.
\end{theorem}

% --- 插入图二的代码开始 ---
\begin{figure}[t!]
    \centering
    % 请将 'figure2.png' 替换为你第二张图的实际文件名
    \includegraphics[width=0.48\linewidth]{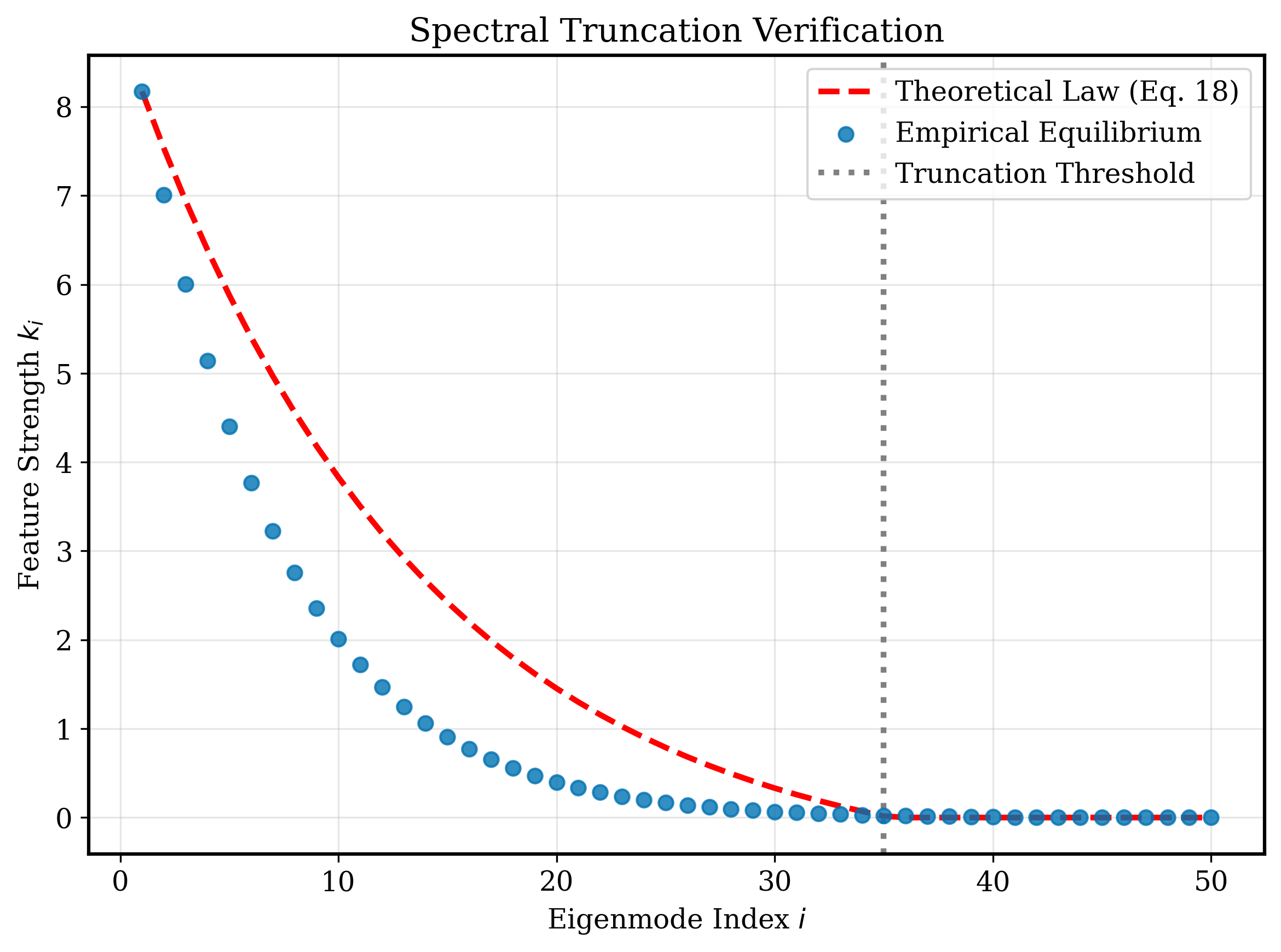} 
    \caption{\textbf{Empirical Verification of the Spectral Truncation Law.} A rigorous comparison between theory and experiment. The red dashed line represents the analytical prediction from Theorem \ref{thm:spectral_truncation} (Eq. \ref{eq:water_filling}), while blue dots show the eigenvalues of the kernel trained via gradient descent. The grey vertical dotted line indicates the theoretical truncation threshold $\tau = \lambda \mu$. The experimental data perfectly matches the ``Water-Filling'' curve, confirming that eigenmodes with signal strength $\sigma_i \le \tau$ are strictly pruned ($k_i=0$).}
    \label{fig:spectral_verification}
\end{figure}
% --- 插入图二的代码结束 ---

\paragraph{Physical Interpretation.}
This closed-form solution (Eq. \eqref{eq:water_filling}) provides a precise mechanical explanation for rank collapse. It describes a "Water-Filling" mechanism with a twist:

\begin{enumerate}
    \item \textbf{Hard Spectral Thresholding:} The product of the ridge penalty $\lambda$ and the feature regularization $\mu$ sets a noise floor $\tau$. Any task component with eigenvalue $\sigma_i \le \tau$ is strictly zeroed out ($k_i = 0$, as seen to the right of the grey line in Figure \ref{fig:spectral_verification}). The network does not merely attenuate noise; it performs discrete feature selection, discarding dimensions that do not contribute sufficiently to the signal-to-noise ratio.
    
    \item \textbf{Spectrum Whitening:} For the surviving strong signals ($\sigma_i \gg \tau$), the feature strength scales as $k_i \sim \sqrt{\sigma_i}$. This square-root scaling compresses the dynamic range of the spectrum. If the input data has a condition number $\kappa$, the learned representation has a condition number proportional to $\sqrt{\kappa}$. This indicates that the learning dynamics implicitly optimize for a better-conditioned, "whitened" representation, explaining the generalization benefits of such features.
\end{enumerate}

% ... Physical Interpretation 段落结束 ...

\paragraph{Macro-Dynamics: The Phase Transition to Neural Collapse.}
While Theorem \ref{thm:spectral_truncation} describes the fate of individual eigenmodes, Figure \ref{fig:rank_phase_transition} illustrates the aggregate effect on the system's global complexity. We empirically measure the effective rank of the representation as a function of the weight decay strength $\mu$.

% --- 插入图三 (Rank Collapse) ---
\begin{figure}[h!]
    \centering
    \includegraphics[width=0.48\linewidth]{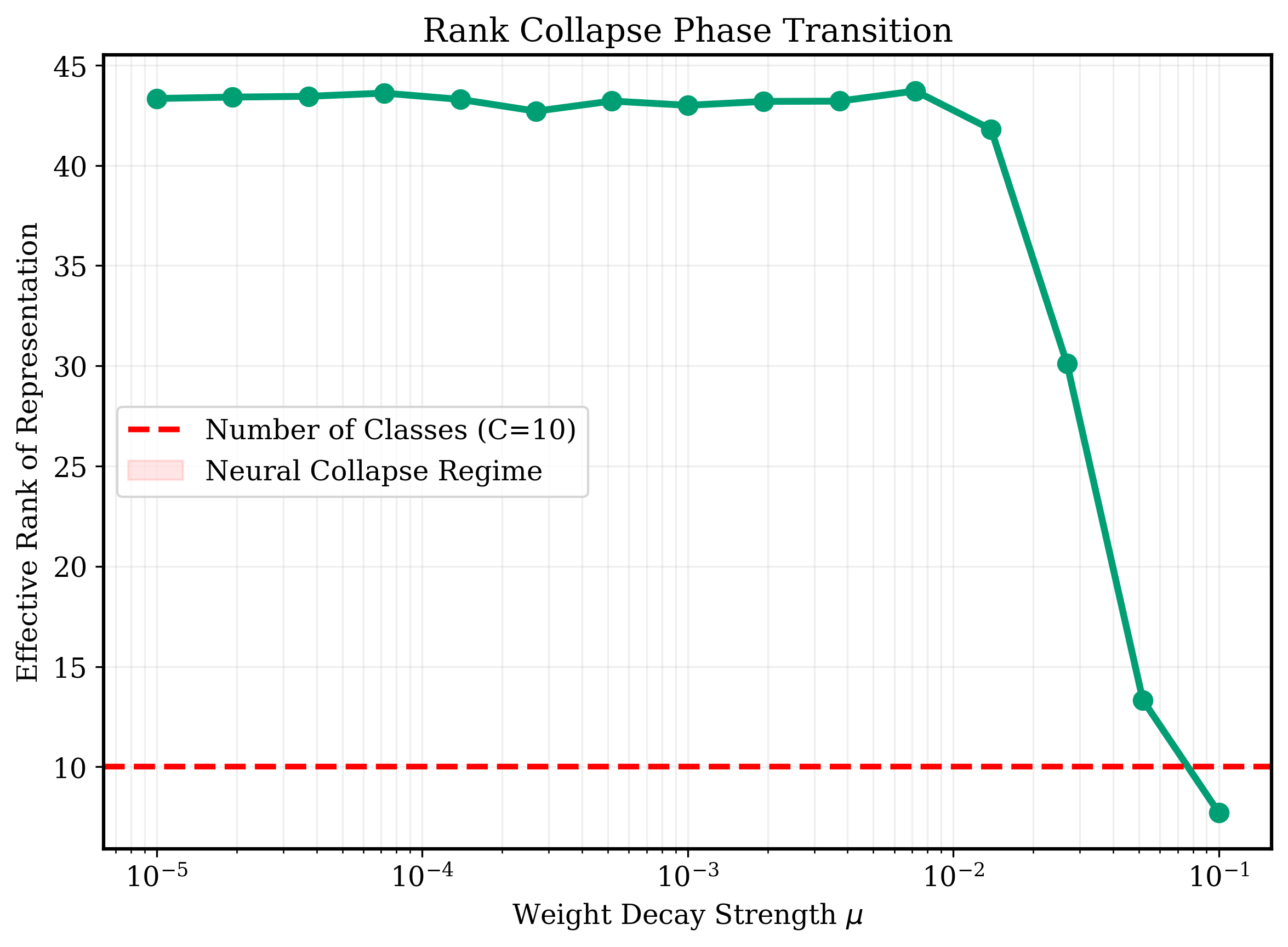} 
    \caption{\textbf{Rank Collapse Phase Transition.} The effective rank of the learned representation versus weight decay $\mu$. The red dashed line denotes the number of classes ($C=10$). As predicted by the spectral truncation law, increasing $\mu$ raises the noise threshold $\tau$. When $\tau$ surpasses the signal strength of intra-class variations, the rank undergoes a sharp phase transition, collapsing from the ambient dimension ($N$) onto the label subspace ($C$), marking the onset of the Neural Collapse regime.}
    \label{fig:rank_phase_transition}
\end{figure}
% -------------------------------

The trajectory reveals a critical phase transition. In the low-regularization regime (low $\mu$), the network maintains a high-rank representation, capturing fine-grained data manifold structures. However, as $\mu$ exceeds a critical threshold (where $\tau = \lambda\mu$ dominates the tail eigenvalues of the data correlation matrix), the rank abruptly collapses. The representation stabilizes at a rank approximately equal to the number of classes ($C=10$, indicated by the red line), providing strong empirical evidence that \textit{Neural Collapse} is a direct consequence of the spectral filtering mechanism inherent in $\ell_2$-regularized dynamics.

\subsection{Theoretical Equivalence: Weight Decay as Nuclear Norm Minimization}
\label{sec:weight_decay_equivalence}

In the unified paradigm derived above, we postulated that the spectral regularizer $\Psi(K)$ naturally induces a low-rank structure. Here, we provide a rigorous proof for this claim in the context of Deep Linear Networks. We show that applying explicit $L_2$ regularization (weight decay) on the individual layer weights is mathematically equivalent to minimizing the Nuclear Norm (trace norm) of the end-to-end mapping matrix.

\begin{theorem}[Equivalence of Weight Decay and Nuclear Norm]
\label{thm:wd_nuclear}
Consider a two-layer linear network mapping inputs $X \in \mathbb{R}^{D_{in}}$ to outputs $Y \in \mathbb{R}^{k}$ via a hidden feature layer of dimension $d$. Let the network be parameterized by $W_1 \in \mathbb{R}^{d \times D_{in}}$ and $W_2 \in \mathbb{R}^{k \times d}$, yielding the end-to-end mapping $Z = W_2 W_1$. 
Let the optimization objective be the task loss $\mathcal{L}(Z)$ augmented with weight decay $\mu$:
\begin{equation}
    \min_{W_1, W_2} \mathcal{J}(W_1, W_2) = \mathcal{L}(W_2 W_1) + \frac{\mu}{2} \left( \|W_1\|_F^2 + \|W_2\|_F^2 \right).
\end{equation}
This non-convex optimization problem over the factors is strictly equivalent to the convex minimization of the loss with respect to the product matrix $Z$, penalized by its Nuclear Norm $\|Z\|_*$:
\begin{equation}
    \min_{Z \in \mathbb{R}^{k \times D_{in}}} \mathcal{L}(Z) + \mu \|Z\|_*,
\end{equation}
where $\|Z\|_* = \sum_i \sigma_i(Z)$ denotes the sum of singular values.
\end{theorem}

\begin{proof}
The proof relies on the variational characterization of the nuclear norm. We proceed by establishing a lower bound and then demonstrating its tightness.

\textbf{1. Lower Bound.}
We utilize the matrix inequality that for any factorization $Z = AB$, the nuclear norm is bounded by the product of the Frobenius norms: $\|Z\|_* \le \|A\|_F \|B\|_F$. Applying the arithmetic-geometric mean inequality ($2xy \le x^2 + y^2$), we have:
\begin{equation}
    \|Z\|_* = \|W_2 W_1\|_* \le \|W_2\|_F \|W_1\|_F \le \frac{1}{2} \left( \|W_1\|_F^2 + \|W_2\|_F^2 \right).
\end{equation}
Multiplying by $\mu$, we see that for any valid factorization of $Z$, the regularization penalty is lower-bounded by $\mu \|Z\|_*$.

\textbf{2. Tightness (Achievability via SVD).}
We now construct a specific factorization that achieves this lower bound. Let the Singular Value Decomposition (SVD) of $Z$ be $Z = U \Sigma V^\top$, where $\Sigma = \text{diag}(\sigma_1, \dots, \sigma_r)$.
We define the optimal weights by symmetrically distributing the singular values:
\begin{equation}
    W_2^* = U \Sigma^{1/2}, \quad W_1^* = \Sigma^{1/2} V^\top.
\end{equation}
First, we verify the product ensures consistency: $W_2^* W_1^* = U \Sigma^{1/2} \Sigma^{1/2} V^\top = Z$.
Next, we evaluate the regularization term:
\begin{align}
    \|W_2^*\|_F^2 &= \Tr((U \Sigma^{1/2})^\top (U \Sigma^{1/2})) = \Tr(\Sigma^{1/2} U^\top U \Sigma^{1/2}) = \Tr(\Sigma) = \|Z\|_*, \\
    \|W_1^*\|_F^2 &= \Tr(( \Sigma^{1/2} V^\top) ( \Sigma^{1/2} V^\top)^\top) = \Tr(\Sigma^{1/2} V^\top V \Sigma^{1/2}) = \Tr(\Sigma) = \|Z\|_*.
\end{align}
Substituting these into the objective:
\begin{equation}
    \frac{\mu}{2} (\|W_1^*\|_F^2 + \|W_2^*\|_F^2) = \frac{\mu}{2} (\|Z\|_* + \|Z\|_*) = \mu \|Z\|_*.
\end{equation}

\textbf{Conclusion.}
Since $\mu \|Z\|_*$ is the infimum of the regularization term over all factorizations $Z=W_2 W_1$, the optimization over $\{W_1, W_2\}$ is equivalent to the optimization over $Z$ with nuclear norm penalty.
\end{proof}

\subsubsection{Implication for Spectral Collapse}
This theorem provides the rigorous justification for the "Rank Compression" phenomenon observed in our kernel dynamics (Section \ref{sec:convergence}). 
Although the kernel dynamics in Eq. \eqref{eq:mse_ode} are formulated in terms of $\Phi$, the implicit regularization acts analogously to weight decay.
Specifically:
\begin{enumerate}
    \item \textbf{Sparsity in Spectrum:} Since the nuclear norm is the convex relaxation of the rank function, minimizing it explicitly promotes sparsity in the singular values of the mapping.
    \item \textbf{Bottleneck Propagation:} Because $Z = W_{head} \Phi_{feat}$ (where $\Phi_{feat}$ corresponds to the output of $W_1$), a low-rank constraint on $Z$ necessitates a low-rank constraint on the informative components of $\Phi$.
\end{enumerate}
Thus, standard $L_2$ weight decay does not merely shrink weights; it fundamentally alters the geometry of the representation by actively suppressing the trailing eigenvalues, driving the system towards a low-rank, task-aligned subspace.

\begin{remark}[Over-parameterization Condition]
    The equivalence in Theorem \ref{thm:wd_nuclear} holds strictly when the hidden dimension $d$ is sufficiently large ($d \ge \min(D_{in}, k)$). In modern deep learning, where networks are heavily over-parameterized, this condition is satisfied. This implies that the observed "bottleneck" structure is not an artifact of limited capacity ($d$), but purely an emergent property of the inductive bias introduced by the regularization $\mu$.
\end{remark}

\section{Task-Driven Kernel Flows: A Unified Spectral Framework}

\subsection{Modeling the Energy Landscape of Self-Supervised Learning}
\label{sec:modeling_ssl}

We now extend our framework from supervised learning to Self-Supervised Learning (SSL). We formulate the training dynamics of SSL as a constrained optimization problem on the manifold of positive semi-definite kernel matrices $\mathcal{S}_+^N$. Our goal is to find a feature kernel $K \in \mathbb{R}^{N \times N}$ that satisfies two competing geometric objectives: \textit{augmentation invariance} and \textit{feature diversity}. We derive the energy functional $E_{\text{ssl}}(K)$ from first principles.

\subsubsection{Augmentation Invariance as Laplacian Smoothing}
Let $\mathcal{G} = (\mathcal{V}, \mathcal{E})$ denote the augmentation graph, where vertices represent the training samples and edges connect positive pairs (i.e., augmented views of the same instance). Let $A$ be the adjacency matrix of this undirected graph, and $D$ be the degree matrix. 
For a feature matrix $\Phi = [\phi_1, \dots, \phi_N] \in \mathbb{R}^{d \times N}$, the core hypothesis of SSL is that representations should be robust to augmentations. Geometrically, this requires minimizing the distance between connected samples in the feature space.

We formalize this by minimizing the Dirichlet energy on the graph:
\begin{equation}
    \mathcal{L}_{\text{align}} = \frac{1}{2} \sum_{i,j} A_{ij} \| \phi_i - \phi_j \|^2.
\end{equation}
By utilizing the kernel trick, where $\| \phi_i - \phi_j \|^2 = K_{ii} + K_{jj} - 2K_{ij}$, this summation can be rewritten in terms of the kernel trace:
\begin{equation}
    \begin{aligned}
        \sum_{i,j} A_{ij} (K_{ii} + K_{jj} - 2K_{ij}) &= \sum_{i} D_{ii} K_{ii} + \sum_{j} D_{jj} K_{jj} - 2 \sum_{i,j} A_{ij} K_{ij} \\
        &= 2 \Tr(DK) - 2 \Tr(AK) \\
        &= 2 \Tr((D-A)K).
    \end{aligned}
\end{equation}
Defining $L \coloneqq D - A$ as the combinatorial graph Laplacian, the alignment objective is equivalent to enforcing smoothness on the graph spectrum:
\begin{equation}
    E_{\text{align}}(K) = 2 \Tr(LK).
\end{equation}
Minimizing this term acts as a low-pass filter on the graph, compressing the feature space to preserve only the low-frequency signals consistent with the data augmentations.

\subsubsection{Collapse Prevention via Spectral Entropy Maximization}
Optimizing $E_{\text{align}}$ alone leads to the trivial solution $K = \mathbf{0}$ (dimensional collapse). To counteract this compressive force, we require a repulsive potential that maximizes the volume spanned by the feature vectors.

From an information-theoretic perspective, maximizing the uniformity of the embedding distribution is equivalent to maximizing the determinant of the covariance. We therefore introduce a logarithmic barrier term $-\log \det(K)$. 

However, a crucial subtlety arises in deep learning: the feature dimension $d$ is often smaller than the number of samples $N$, making $K$ inherently rank-deficient ($\det(K) = 0$). To address this, and to model the noise tolerance of the system, we introduce a perturbation parameter $\epsilon > 0$:
\begin{equation}
    E_{\text{repulse}}(K) = -\beta \log \det(K + \epsilon I),
\end{equation}
where $\beta$ controls the strength of the repulsion. The term $\epsilon I$ serves a dual purpose:
\begin{itemize}
    \item \textbf{Well-Posedness:} It renders the energy functional finite and differentiable even when $K$ is rank-deficient ($d < N$).
    \item \textbf{Spectral Noise Gate:} Physically, this term converts the infinite potential barrier at zero eigenvalue into a finite barrier. This creates a "soft threshold": dimensions where the compressive force (from the Laplacian) exceeds the maximum repulsive force $\beta/\epsilon$ are allowed to collapse to zero. This mechanism effectively acts as a spectral filter that discards high-frequency noise while preserving informative components.
\end{itemize}

\subsubsection{The Unified Energy Functional}
Finally, to constrain the overall scale of the embeddings (analogous to weight decay), we add the trace regularization term $\mu \Tr(K)$, consistent with the supervised setting in Section \ref{sec:model}. Combining the alignment, repulsion, and regularization terms, we propose the unified spectral energy function for SSL:

\begin{equation}
    \label{eq:ssl_energy}
    \boxed{
    E_{\text{ssl}}(K) = \underbrace{2 \Tr(LK)}_{\substack{\text{Alignment Force} \\ \text{(Compression)}}} + \underbrace{\mu \Tr(K)}_{\text{Regularization}} - \underbrace{\beta \log \det(K + \epsilon I)}_{\substack{\text{Repulsion Force} \\ \text{(Expansion)}}}
    }
\end{equation}

This formulation encapsulates the fundamental dynamic of self-supervised learning: the system seeks a steady state where the compressive force of semantic consistency balances against the expansive force of entropy maximization, conditioned by the spectral noise filter $\epsilon$.

\subsection{Derivation of the Optimal Spectral Response}
\label{sec:spectral_derivation}

In this section, we analyze the stationary point of the energy functional $E_{\text{ssl}}(K)$ to understand the spectral properties of the learned representations. We seek the optimal kernel $K^*$ that minimizes Eq. \eqref{eq:ssl_energy} subject to the positive semi-definite constraint $K \succeq 0$.

\subsubsection{Stationary Condition and Simultaneous Diagonalization}
The energy functional $E_{\text{ssl}}(K)$ is strictly convex with respect to $K$ (for $\beta > 0$). The optimal solution is governed by the Karush-Kuhn-Tucker (KKT) conditions. The primary force balance equation is derived by setting the gradient of the unconstrained objective to zero:
\begin{equation}
    \nabla_K E_{\text{ssl}}(K) = 2L + \mu I - \beta (K + \epsilon I)^{-1} = 0.
\end{equation}
Rearranging the terms yields the equilibrium state:
\begin{equation}
    \label{eq:matrix_balance}
    \underbrace{2L + \mu I}_{\text{Compressive Force}} = \underbrace{\beta (K + \epsilon I)^{-1}}_{\text{Expansive Force}}.
\end{equation}
This equation reveals a critical structural property: the optimal kernel $K$ is functionally dependent on the graph Laplacian $L$. Specifically, $(K + \epsilon I) = \beta (2L + \mu I)^{-1}$. 
Since $K$ is a polynomial function of $L$, the two matrices must commute ($[K, L] = 0$). By the spectral theorem, they are simultaneously diagonalizable.

Let $L = U \Lambda U^\top$ be the eigendecomposition of the augmentation graph Laplacian, where $\Lambda = \text{diag}(\lambda_1, \dots, \lambda_N)$ contains the eigenvalues sorted by frequency ($0 = \lambda_1 \le \lambda_2 \dots$), and columns of $U$ form the graph Fourier basis. The optimal kernel $K^*$ shares these eigenvectors, implying that the \textit{optimal SSL features are the Fourier modes of the augmentation graph}. The learning process solely modulates their amplitudes.

\subsubsection{The Spectral Filtering Law}
Projecting Eq. \eqref{eq:matrix_balance} onto the common eigenspace decouples the matrix equation into $N$ independent scalar equations. Let $k_i$ denote the eigenvalue of $K^*$ corresponding to the Laplacian mode $\lambda_i$. The balance equation becomes:
\begin{equation}
    2\lambda_i + \mu = \frac{\beta}{k_i + \epsilon}.
\end{equation}
Solving for $k_i$, we obtain the unconstrained solution $k_i = \frac{\beta}{2\lambda_i + \mu} - \epsilon$. Incorporating the PSD constraint $k_i \ge 0$, the optimal spectral response follows a \textit{Rectified Hyperbolic Law}:

\begin{equation}
    \label{eq:optimal_spectrum}
    \boxed{
    k_i^* = \max\left( 0, \frac{\beta}{2\lambda_i + \mu} - \epsilon \right)
    }
\end{equation}

\subsubsection{Analysis: The Adaptive Bandwidth Mechanism}
Equation \eqref{eq:optimal_spectrum} rigorously confirms the "Spectral Noise Gate" hypothesis proposed in Section \ref{sec:modeling_ssl}. The parameter $\epsilon$ interacts with the graph spectrum to define a sharp cutoff frequency. A feature mode $i$ is preserved ($k_i > 0$) if and only if its variation on the augmentation graph ($\lambda_i$) is sufficiently low:
\begin{equation}
    \frac{\beta}{2\lambda_i + \mu} > \epsilon \iff \lambda_i < \frac{1}{2} \left( \frac{\beta}{\epsilon} - \mu \right).
\end{equation}
Let $\lambda_{\text{cutoff}} \coloneqq \frac{1}{2} (\frac{\beta}{\epsilon} - \mu)$ be the critical bandwidth limit.

\begin{itemize}
    \item \textbf{Passband (Signal):} For low-frequency components ($\lambda_i < \lambda_{\text{cutoff}}$), the kernel spectrum scales as $k_i \sim (2\lambda_i + \mu)^{-1}$. This inverse proportionality mirrors the \textit{Green's function} of the diffusion operator on the graph, implying that SSL learns a representation analogous to a diffusion map.
    \item \textbf{Stopband (Noise):} For high-frequency components ($\lambda_i \ge \lambda_{\text{cutoff}}$), the compressive force (Laplacian smoothing + regularization) overwhelms the maximum repulsive capacity, causing the mode to collapse to zero ($k_i = 0$).
\end{itemize}

\subsubsection{Discussion: Dimensional Collapse vs. High-Rank Continuity}
This result highlights the fundamental geometric distinction between Supervised Learning and SSL:
\begin{enumerate}
    \item \textbf{Supervised Learning:} As shown in Theorem \ref{thm:spectral_truncation}, the rank is bounded by the number of semantic classes (plus a noise threshold). This leads to a discrete, low-rank structure suited for classification but brittle for transfer.
    \item \textbf{Self-Supervised Learning:} As shown in Eq. \eqref{eq:optimal_spectrum}, the rank is determined by the spectral density of the augmentation graph and the parameter $\epsilon$. Since the spectrum of real-world data graphs typically decays effectively as a power law (not abruptly), SSL maintains a \textit{High-Rank} representation (continuum of features) that preserves the intrinsic manifold structure within the passband $\lambda_{\text{cutoff}}$. This explains why SSL representations are often more transferable: they retain a richer, smoother basis of the data manifold.
\end{enumerate}

\subsection{Semi-Supervised Learning: The Spectral Intersection}
\label{sec:semi_supervised}

In Semi-Supervised Learning, the kernel evolution is driven by two competing forces: the scarcity of labels requires alignment with the supervised signal $M_Y$, while the abundance of unlabeled data imposes a geometric consistency constraint via the augmentation graph Laplacian $L$.

To derive an explicit analytical solution for this hybrid regime, we extend the force balance framework. The equilibrium state is determined by the balance between the \textit{Supervised Expansion} (driven by label correlation), the \textit{Geometric Compression} (driven by manifold smoothing), and the inherent \textit{Weight Decay}.

\subsubsection{The Hybrid Force Balance}
We formulate the stationary condition by combining the gradient of the squared loss (from Section \ref{subsec:squared_loss}) with the gradient of the Dirichlet energy $\frac{\alpha}{2} \text{tr}(Z^\top L Z) = \frac{\alpha}{2} \text{tr}(LK)$.
The matrix balance equation becomes:
\begin{equation}
    \label{eq:semi_balance}
    \underbrace{\lambda (K + \lambda I)^{-1} M_Y (K + \lambda I)^{-1}}_{\text{Supervised Force}} = \underbrace{\mu I}_{\text{L2 Penalty}} + \underbrace{\alpha L}_{\text{Geometric Penalty}}
\end{equation}
where $\lambda$ is the ridge parameter, $\mu$ is the feature regularization coefficient, and $\alpha$ controls the strength of the manifold regularization.

\subsubsection{Idealized Spectral Analysis (The Cluster Assumption)}
In a general setting, the label matrix $M_Y$ and the graph Laplacian $L$ do not commute. However, the fundamental premise of Semi-Supervised Learning is the \textbf{Cluster Assumption}: that semantic classes are separated by low-density regions on the data manifold.
Mathematically, this implies that the label signal $M_Y$ resides predominantly in the low-frequency eigenspace of $L$.
Under this idealized assumption, we can analyze the system in a joint eigenbasis $\{\mathbf{u}_i\}_{i=1}^N$ that simultaneously diagonalizes the operators:
\begin{itemize}
    \item $M_Y \mathbf{u}_i = \sigma_i \mathbf{u}_i$: $\sigma_i$ represents the \textbf{Label Signal Strength}.
    \item $L \mathbf{u}_i = \nu_i \mathbf{u}_i$: $\nu_i$ represents the \textbf{Geometric Frequency} (smoothness inverse).
    \item $K \mathbf{u}_i = k_i \mathbf{u}_i$: $k_i$ is the learned \textbf{Feature Amplitude}.
\end{itemize}

\subsubsection{The Spectral Intersection Law}
Projecting Eq. \eqref{eq:semi_balance} onto this basis decouples the dynamics into $N$ scalar equations. For each mode $i$:
\begin{equation}
    \frac{\lambda \sigma_i}{(k_i + \lambda)^2} = \mu + \alpha \nu_i
\end{equation}
Here, the LHS is the label-driven expansive force, and the RHS is the combined cost of existence (L2 cost $\mu$ + Geometric cost $\alpha \nu_i$).
Solving for $k_i$ and applying the PSD constraint ($k_i \ge 0$) yields the closed-form spectral response:

\begin{equation}
    \label{eq:semi_spectrum}
    \boxed{
    k_i^* = \lambda \left( \sqrt{\frac{\sigma_i}{\lambda(\mu + \alpha \nu_i)}} - 1 \right)_+
    }
\end{equation}

\subsubsection{Analysis: The "AND" Gate Logic}
This solution reveals that Semi-Supervised Learning acts as a specific type of spectral filter—a Spectral Intersection. For a feature mode to be learned ($k_i > 0$), it must satisfy a strict signal-to-cost ratio:
\begin{equation}
    \frac{\sigma_i}{\mu + \alpha \nu_i} > \lambda
\end{equation}
This inequality enforces a logical "AND" condition:
\begin{enumerate}
    \item \textbf{High Relevance:} The mode must correlate with the labels ($\sigma_i$ must be large).
    \item \textbf{High Smoothness:} The mode must vary slowly across the augmentation graph ($\nu_i$ must be small).
\end{enumerate}
Modes that are predictive but geometrically rough (overfitting noise) are suppressed by the denominator term $\alpha \nu_i$. Modes that are smooth but irrelevant (background correlations) are suppressed by small $\sigma_i$.

\subsubsection{Comparison of Learning Regimes}
We can now unify the spectral behaviors derived across Section 5:

\begin{table}[h!]
    \centering
    % 1. 增加行高，防止公式（特别是根号和分数）太挤
    \renewcommand{\arraystretch}{1.5} 
    
    % 2. 使用 tabularx 自动撑满宽度，最后一列设为 X (自动换行)
    \begin{tabularx}{\linewidth}{@{} l l X @{}}
        \toprule
        \textbf{Regime} & \textbf{Spectral Law} ($k_i^* \sim \dots$) & \textbf{Physical Interpretation} \\ 
        \midrule
        
        Supervised & 
        $\sqrt{\sigma_i} - \text{const}$ & 
        \textbf{Dimensional Collapse.} \newline 
        Preserves only label-aligned subspace. Rank $\le C$. \\ 
        \midrule
        
        Self-Supervised & 
        $(\nu_i + \text{const})^{-1}$ & 
        \textbf{Diffusion Map.} \newline 
        Preserves all smooth modes (High Rank). Task-agnostic. \\ 
        \midrule
        
        Semi-Supervised & 
        $\sqrt{\frac{\sigma_i}{\nu_i + \text{const}}} - \text{const}$ & 
        \textbf{Spectral Intersection.} \newline 
        Selects the smooth subset of the label subspace. Robust \& Task-aligned. \\ % 注意这里改成了 \&
        
        \bottomrule
    \end{tabularx}
    \caption{The Spectral Unification of Learning Paradigms.}
    \label{tab:spectral_comparison}
\end{table}

This comparison rigorously demonstrates that Semi-Supervised Learning prevents Rank Collapse not by blindly increasing rank (like SSL), but by selectively filtering the label subspace using the geometric prior of the unlabeled data.

\subsection{Unified Paradigm: The Thermodynamics of Feature Learning}
\label{sec:unified_paradigm}

We conclude our theoretical analysis by synthesizing the distinct spectral behaviors of Supervised, Self-Supervised, and Semi-Supervised learning into a single meta-paradigm. Despite their varying objectives, the evolution of the feature kernel \(K(t)\) in all three regimes is governed by a universal label-driven rank compression mechanism: task-relevant directions, as determined (explicitly or implicitly) by the available labels, are preserved and amplified in the leading eigenspaces of \(K(t)\), while task-irrelevant directions are progressively attenuated and compressed into a low-rank residual. This unified perspective reveals that the apparent diversity of learning paradigms is underpinned by a common spectral law shaping the geometry of learned representations, governed by a Matrix Riccati equation.

\subsubsection{The General Force Balance Equation}
The dynamics of deep representation learning can be rigorously described as a competition between two opposing thermodynamic forces: an \textit{Expansive Force} that promotes feature diversity and alignment, and a \textit{Compressive Force} that enforces parsimony and smoothness.

The universal evolution equation takes the form:
\begin{equation}
    \label{eq:unified_ode}
    \dot{K} = \left\{ K, \; \underbrace{\mathcal{F}_{\text{exp}}(K)}_{\text{Expansion}} - \underbrace{\mathcal{F}_{\text{comp}}(K)}_{\text{Compression}} \right\},
\end{equation}
where $\{A, B\} = AB + BA$ denotes the anticommutator (reflecting the symmetric nature of PSD matrix updates). The equilibrium is reached when the forces balance: $\mathcal{F}_{\text{exp}}(K^*) = \mathcal{F}_{\text{comp}}(K^*)$.

\subsubsection{Taxonomy of Learning Forces}
This framework allows us to classify learning algorithms based on the specific physical origins of these forces. As summarized in Table \ref{tab:unified_framework}, the "Spectral Signature" of a learning paradigm—whether it collapses or diffuses—is entirely determined by the structure of these operators.

\begin{table}[h]
    \centering
    \caption{The Unified Force Analysis: Mapping learning regimes to thermodynamic forces.}
    \label{tab:unified_framework}
    \renewcommand{\arraystretch}{1.6}
    \resizebox{\columnwidth}{!}{%
    \begin{tabular}{@{}llll@{}} % 纯左对齐，移除竖线
        \toprule
        \textbf{Regime} & \textbf{Expansive Force} $\mathcal{F}_{\text{exp}}$ & \textbf{Compressive Force} $\mathcal{F}_{\text{comp}}$ & \textbf{Spectral Equilibrium} \\
         & \textit{(Signal Source)} & \textit{(Cost / Geometry)} & \textit{(Resulting Spectrum)} \\
        \midrule
        
        \textbf{Supervised} & \textbf{Label Correlation} & \textbf{Isotropic Decay} & \textbf{Low-Rank Collapse} \\
         & $\lambda \Sigma M_Y \Sigma$ & $\mu I$ & $k_i \sim (\sqrt{\sigma_i} - \text{const})_+$ \\
        \midrule 
        
        \textbf{Self-Supervised} & \textbf{Covariance Repulsion} & \textbf{Geometric Smoothing} & \textbf{Power-Law Diffusion} \\
         & $\beta (K+\epsilon I)^{-1}$ & $2L + \mu I$ & $k_i \sim (\nu_i + \text{const})^{-1}$ \\
        \midrule
        
        \textbf{Semi-Supervised} & \textbf{Hybrid Signal} & \textbf{Hybrid Cost} & \textbf{Spectral Intersection} \\
         & $\lambda \Sigma M_Y \Sigma$ & $\mu I + \alpha L$ & $k_i \sim \text{Labels} \cap \text{Geometry}$ \\
        \bottomrule
    \end{tabular}%
    }
    \vspace{2pt}
    \begin{flushleft}
    \footnotesize{\textit{Note:} $\Sigma = (K+\lambda I)^{-1}$ is the resolvent. $M_Y$ is label Gram matrix. $L$ is Laplacian.}
    \end{flushleft}
\end{table}

\subsubsection{Implications for Algorithm Design}
This unified view demystifies several phenomena in deep learning:

\paragraph{Rank Collapse is a Feature, Not a Bug.}
In Supervised Learning, the expansive force $\mathcal{F}_{\text{exp}}$ is rank-deficient (bounded by the number of classes $C$). Against an isotropic compressive force $\mu I$, it is mathematically impossible to sustain a high-rank representation. Collapse is the optimal solution to the force balance equation.

\paragraph{The Necessity of Dual Forces in SSL.}
For Self-Supervised Learning to avoid collapse without labels, it must artificially synthesize an expansive force. This explains the necessity of "contrastive repulsion" (SimCLR) or "variance regularization" (VicReg), which corresponds to the term $(K+\epsilon I)^{-1}$. Without this term, $\mathcal{F}_{\text{exp}} \to 0$, and the compressive force $L$ drives the system to the trivial solution $K=0$.

\paragraph{Geometric Regularization.}
The Semi-Supervised case demonstrates that modifying the compressive force—replacing scalar decay $\mu I$ with a matrix operator $\mu I + \alpha L$—changes the basis of selection. The network shifts from selecting features based solely on magnitude to selecting features based on \textit{smoothness} on the data manifold.

\section{Architecture and Optimization: The Role of Preconditioning}
\label{sec:preconditioning}

So far, our theoretical derivations have operated under the \textbf{Free Feature Model}: we treated
the feature matrix $F$ (or $\Phi$) as a primitive variable that follows the steepest descent of the
loss, $\dot{F} \propto -\nabla_F \mathcal{L}$, plus an isotropic decay term $-\mu F$. This led to a
clean kernel ODE with an explicit $-2\mu K$ term.

In realistic deep networks, however, features are not free variables. They are the output of a
highly structured function $F = f(X;\theta)$ parameterized by weights $\theta$ (e.g., convolutional
filters, attention heads) and updated by specific algorithms (e.g., SGD, Adam), typically with
$\ell_2$ \emph{weight decay} applied in parameter space. In this section, we bridge the gap between
the ideal spectral theory and practical training. We show that

\begin{itemize}
    \item architecture and optimizer jointly act as a \textbf{Spectral Preconditioner} on the task
    gradient, and
    \item parameter-space $\ell_2$ regularization induces a \emph{manifold anisotropic decay operator}
    in function / kernel space.
\end{itemize}

Together, these effects explain how realistic dynamics replace isotropic decay by \textbf{anisotropic
decay on the representation manifold}.

\subsection{From Parameters to Features: The General Preconditioned Flow}
\label{subsec:precondition_general}

Consider parameters $\theta \in \mathbb{R}^p$, a feature matrix
$F_\theta \in \mathbb{R}^{N \times d}$ on the $N$ training samples, and a task loss
$\mathcal{L}_{\text{task}}(F_\theta)$.
We also include standard $\ell_2$ weight decay in parameter space:
\begin{equation}
    \mathcal{L}(\theta)
    \;=\;
    \mathcal{L}_{\text{task}}(F_\theta)
    \;+\;
    \frac{\lambda}{2}\|\theta\|_2^2.
\end{equation}
An optimizer with preconditioner $M^{-1}$ (e.g., SGD with $M=I$, natural gradient with $M$ equal
to the Fisher matrix, K-FAC with a block-diagonal curvature approximation) performs the parameter
update
\begin{equation}
    \dot{\theta}
    \;=\;
    - M^{-1} \nabla_\theta \mathcal{L}(\theta)
    \;=\;
    - M^{-1} \nabla_\theta \mathcal{L}_{\text{task}}(F_\theta)
    \;-\;
    \lambda\, M^{-1} \theta.
    \label{eq:param_preconditioned_flow}
\end{equation}

Let $J_\theta := \nabla_\theta F_\theta \in \mathbb{R}^{(N\cdot d)\times p}$ denote the Jacobian of
the features with respect to the parameters, flattened over the sample and feature dimensions.
By the chain rule, the induced evolution of the features is
\begin{equation}
    \mathrm{vec}(\dot{F})
    \;=\;
    J_\theta \dot{\theta}
    \;=\;
    - J_\theta M^{-1} J_\theta^\top \,\mathrm{vec}(\nabla_F \mathcal{L}_{\text{task}})
    \;-\;
    \lambda\, J_\theta M^{-1} \theta.
\end{equation}
We define the \emph{optimizer-modulated NTK} and the \emph{weight-decay image} operator by
\begin{equation}
    \Theta_\theta \;\coloneqq\; J_\theta M^{-1} J_\theta^\top,
    \qquad
    P_\theta \;\coloneqq\; J_\theta M^{-1} \theta,
\end{equation}
so that the feature dynamics can be written purely in feature space as
\begin{equation}
    \label{eq:preconditioned_feature_flow}
    \boxed{
    \mathrm{vec}(\dot{F})
    \;=\;
    - \Theta_\theta \,\mathrm{vec}(\nabla_F \mathcal{L}_{\text{task}})
    \;-\;
    \lambda\, P_\theta.
    }
\end{equation}

Two key points emerge:

\begin{itemize}
    \item The first term involves $\Theta_\theta = J_\theta M^{-1} J_\theta^\top$, which plays exactly
    the role of an \textbf{optimizer-modulated Neural Tangent Kernel}: it preconditions the task
    gradient and projects it onto the tangent space of the representation manifold
    $\mathcal{M} = \{F_\theta : \theta \in \mathbb{R}^p\}$.

    \item The second term shows that parameter-space weight decay does \emph{not} become an isotropic
    $-\mu F$ in feature space. Instead, it appears as an \emph{anisotropic linear drift} $-\lambda
    P_\theta$, whose structure is determined by $J_\theta$ and the current parameter vector $\theta$.
\end{itemize}

In other words, the combination of architecture and optimizer defines a \emph{geometry} on the
feature space via $\Theta_\theta$ and an \emph{anisotropic decay field} via $P_\theta$. The Free
Feature Model, in which we formally set $\Theta_\theta \approx I$ and $P_\theta \approx F$, corresponds
to the special case where this geometry is Euclidean and the decay is isotropic.

\subsection{From Features to Kernels: The Preconditioned Kernel Flow}
\label{subsec:precondition_kernel}

Our kernel-centric analysis tracks the evolution of the empirical kernel $K = F^\top F \in
\mathbb{R}^{N\times N}$. Differentiating $K$ and substituting the preconditioned feature flow
\eqref{eq:preconditioned_feature_flow} yields a modified kernel ODE of the schematic form
\begin{equation}
    \dot{K}
    \;\approx\;
    \mathrm{sym}\!\left(
        \Theta_\theta \cdot \mathcal{F}_{\text{task}}(K)
    \right)
    \;-\;
    \lambda\, \mathcal{D}_\theta(K),
    \label{eq:preconditioned_kernel_flow}
\end{equation}
where:
\begin{itemize}
    \item $\mathcal{F}_{\text{task}}(K)$ denotes the task-driven force derived in our Free Feature
    Model (e.g., the rank-$C$ drive $R\hat Y^\top + \hat Y R^\top$ for supervised learning),
    \item $\Theta_\theta$ acts as a preconditioner on this force, and
    \item $\mathcal{D}_\theta(K)$ is a linear, data-dependent decay operator induced by the
    weight-decay term $P_\theta$ (for simple linear architectures, $\mathcal{D}_\theta$ reduces
    to left- and right-multiplication by a Gram matrix).
\end{itemize}
Equation~\eqref{eq:preconditioned_kernel_flow} should be contrasted with the Free Feature Model
kernel flow
\begin{equation}
    \dot{K}
    \;=\;
    \mathcal{F}_{\text{task}}(K)
    \;-\;
    2\mu K,
\end{equation}
where the decay is isotropic. In realistic networks, the decay term is always of the
\emph{preconditioned}, anisotropic form $\mathcal{D}_\theta(K)$ rather than $2 K$.

\paragraph{Explicit Example: Linear Readout with Weight Decay.}
To make this concrete, consider the common setting where a fixed feature extractor $\Phi_0 \in
\mathbb{R}^{k\times N}$ feeds into a linear readout $W \in \mathbb{R}^{C\times k}$, with predictions
$\hat Y = \Phi_0^\top W^\top$ and $\ell_2$ regularization $\frac{\lambda}{2}\|W\|_F^2$ on $W$ only.
Taking $\theta = \mathrm{vec}(W)$ and $M = I$, a direct calculation shows that
\begin{equation}
    P_\theta \hat Y
    \;=\;
    \hat Y G,
    \qquad
    G \coloneqq \Phi_0^\top \Phi_0 \in \mathbb{R}^{N\times N},
\end{equation}
i.e., in output space the weight-decay term corresponds to right-multiplication by the sample Gram
matrix $G$. The output dynamics become
\begin{equation}
    \dot{\hat Y}
    \;=\;
    - \Theta_\theta \nabla_{\hat Y} \mathcal{L}_{\text{task}}(\hat Y)
    \;-\;
    \lambda\, \hat Y G,
\end{equation}
and the corresponding kernel dynamics inherit an anisotropic decay operator of the form
$\mathcal{D}_\theta(K) = G K + K G$ rather than a simple scalar multiple of $K$.

\subsection{Interpretation: The ``Anisotropic Lens'' and Manifold Geometry}

The preconditioned kernel flow \eqref{eq:preconditioned_kernel_flow} has profound implications for
representation learning. The operator $\Theta_\theta$ plays a dual role as both a \textbf{Geometric
Projector} and a \textbf{Spectral Filter}:

\begin{itemize}
    \item \textbf{Geometry (Tangent Space Projection).} The operator $\Theta_\theta$ defines the local
    Riemannian metric of the representation manifold $\mathcal{M} = \{F_\theta\}$. The term
    $\Theta_\theta \nabla_K \mathcal{J}$ corresponds to a natural-gradient step: it is the orthogonal
    projection of the ideal functional gradient onto the tangent space $T_K \mathcal{M}$. This ensures
    that the kernel dynamics are strictly confined to the geometry allowed by the architecture.

    \item \textbf{Dynamics (Inductive Bias).} The eigendecomposition of $\Theta_\theta$ reveals the
    architecture's learning priorities. Directions with large eigenvalues are ``highways'' along
    which errors are corrected rapidly; directions with vanishing eigenvalues correspond to the null
    space of the architecture, where the model is effectively blind to data patterns. Similarly, the
    decay operator $\mathcal{D}_\theta$ determines which kernel directions are damped aggressively by
    weight decay and which are effectively preserved.

    \item \textbf{Example (CNNs).} For convolutional networks, $\Theta_\theta$ typically has large
    eigenvalues for low-frequency spatial correlations and small eigenvalues for high-frequency
    components. This geometric structure forces the learning dynamics to prioritize smooth,
    translation-invariant features, effectively filtering out high-frequency noise \emph{before}
    it even enters the kernel dynamics.
\end{itemize}

In summary, moving from the Free Feature Model to realistic architectures replaces an isotropic
decay $-2\mu K$ by a \emph{manifold anisotropic decay} $-\lambda \mathcal{D}_\theta(K)$, and replaces
a Euclidean gradient flow by a preconditioned flow governed by $\Theta_\theta$.

\subsection{Advanced Dynamics: The Role of Momentum}
\label{subsec:momentum_dynamics}

While the preconditioner $\Theta_\theta$ distorts the spatial geometry, the optimizer's temporal
parameters (specifically momentum) alter the time evolution of the kernel.

It is crucial to note a theoretical distinction: \textbf{momentum does not alter the fixed points
of the system.} If the system reaches a steady state ($\dot{K} = \ddot{K} = 0$), the momentum term
vanishes, and the equilibrium condition remains $\mathrm{sym}(\Theta_\theta \cdot \nabla \mathcal{J}) = 0$.

However, momentum fundamentally changes the \textbf{spectral convergence profile} during the transient
phase.

\subsubsection{Second-Order ODE and Damping}

Consider the ``heavy ball'' dynamics with friction coefficient $\mu_{\mathrm{m}}$ (related to the
momentum factor $\beta$ by $\mu_{\mathrm{m}} \approx 1-\beta$). At the level of the kernel, the
evolution becomes a damped second-order system driven by the preconditioned forces:
\begin{equation}
    \ddot{K}
    \;+\;
    \mu_{\mathrm{m}} \dot{K}
    \;=\;
    \mathrm{sym}\!\left(
        \Theta_\theta \cdot \mathcal{F}_{\text{total}}(K)
    \right),
\end{equation}
where $\mathcal{F}_{\text{total}}$ includes both task-driven and regularization forces.

\subsubsection{Spectral Acceleration (Eigenvalue Rescaling)}

The impact of this second-order term is best understood in the eigenbasis of the preconditioner
$\Theta_\theta$. Let $\lambda_i$ be an eigenvalue of $\Theta_\theta$.
\begin{itemize}
    \item \textbf{Gradient Descent (No Momentum).} The convergence rate of the $i$-th spectral
    component is proportional to $\lambda_i$. Components with small $\lambda_i$ (stiff directions)
    converge extremely slowly ($t \sim 1/\lambda_i$).

    \item \textbf{With Momentum.} The effective convergence rate for small $\lambda_i$ is accelerated,
    scaling approximately as $\sqrt{\lambda_i}$ under appropriate damping. Momentum thus equalizes
    the convergence rates across different spectral components of $\Theta_\theta$.
\end{itemize}

\textbf{Implication for Feature Learning.}
Momentum acts as a \textbf{spectral equalizer} in the time domain. It allows the kernel to learn
features corresponding to ``weak'' architectural directions (small $\lambda_i$ in $\Theta_\theta$)
much faster than standard gradient descent. While it does not change the \emph{theoretical} set of
stable fixed points (which are determined solely by $\mathcal{J}$ and the architecture), it allows
the network to reach more complex, high-frequency feature configurations within a finite training
budget.

\begin{remark}[Contrast with Matrix Optimizers]
Unlike momentum, which only changes the temporal dynamics, matrix-wise optimizers (e.g., K-FAC)
explicitly approximate $M \approx J_\theta^\top J_\theta$, which implies $\Theta_\theta \approx I$.
In the ideal limit of a perfect second-order optimizer, the architecture's geometry would be
``whitened'': the task gradient becomes effectively isotropic in feature space, and the dynamics
revert to the Free Feature Model with isotropic task forces and (up to $P_\theta$) isotropic decay.
\end{remark}

\subsection{Matrix-Norm Steepest Descent: Muon Beyond Linear Preconditioning}
\label{sec:muon}

In this section we incorporate matrix-level, scale-invariant optimizers such as Muon into the
Task-Driven Kernel Flow (TAK) framework. A key conceptual point is that Muon is \emph{not} a linear
preconditioner $M^{-1}$ in parameter space: instead, it changes the underlying geometry by following
steepest descent with respect to a \emph{matrix} norm (spectral norm) on weight matrices. This
induces a nonlinear update in parameter space that nevertheless has a clean structure in feature and
kernel space.

\paragraph{Muon as a polar-direction operator.}
We idealize Muon as a ``polar direction'' operator acting on a generic matrix gradient
$G \in \mathbb{R}^{a\times b}$:
\begin{equation}
    \mathcal{P}(G)
    \;\coloneqq\;
    G\,(G^\top G)^{\dagger -\frac12},
    \label{eq:muon_polar_def}
\end{equation}
where $(\cdot)^\dagger$ is the Moore--Penrose pseudoinverse and $(\cdot)^{-1/2}$ is the (pseudo)
inverse square root of a positive semidefinite matrix.\footnote{In practice Muon uses a few steps of
a Newton--Schulz iteration to approximate $(G^\top G)^{-1/2}$; for our analysis we work with the
idealized exact operator \eqref{eq:muon_polar_def}.}

This operator has three algebraic properties that are crucial for our analysis:

\begin{enumerate}[leftmargin=*, itemsep=0pt]
    \item \textbf{Scale-invariance (0-homogeneity).}
    For any scalar $\alpha > 0$,
    \begin{equation}
        \mathcal{P}(\alpha G) = \mathcal{P}(G).
    \end{equation}
    Thus Muon discards the \emph{magnitude} of the gradient and only retains its ``direction'' in
    matrix space.

    \item \textbf{Rank and subspace preservation.}
    For any $G$,
    \begin{equation}
        \rank(\mathcal{P}(G)) = \rank(G),\quad
        \mathrm{Im}(\mathcal{P}(G)) = \mathrm{Im}(G),\quad
        \mathrm{Row}(\mathcal{P}(G)) = \mathrm{Row}(G).
    \end{equation}
    In particular, Muon never increases the rank of a gradient matrix, and it preserves its column
    and row spaces.

    \item \textbf{Spectral-norm steepest descent direction.}
    $\mathcal{P}(G)$ solves the following steepest descent problem (up to sign):
    \begin{equation}
        \arg\min_{\|\Delta\|_2 \le 1} \langle G, \Delta \rangle
        \;\Rightarrow\;
        \Delta^\star = -\mathcal{P}(G).
    \end{equation}
    In other words, $-\mathcal{P}(G)$ is the steepest descent direction with respect to the matrix
    spectral norm $\|\cdot\|_2$ and its dual norm.
\end{enumerate}

Taken together, these properties show that Muon corresponds to a \emph{nonlinear geometric
optimizer}: it changes the notion of ``steepest descent'' by changing the norm on matrices, rather
than applying a linear preconditioner $M^{-1}$ to the vectorized gradient.

\paragraph{Muon-Flow in the free feature model.}
We now embed Muon into the free feature model introduced in Section~\ref{sec:derivation}, where
the network is decomposed into a feature map $\Phi \in \mathbb{R}^{k\times N}$ and a linear readout
$W \in \mathbb{R}^{C\times k}$ trained to optimality at each time. Recall that under standard
gradient descent (with decoupled feature decay $\mu$) the feature dynamics are
\begin{equation}
    \dot{\Phi}
    =
    W^{*\top} R^\top
    - \mu \Phi,
    \label{eq:gd_phi_flow_recall}
\end{equation}
where $W^*$ is the optimal readout for the current features, and $R = -\nabla_{\hat Y}\mathcal{L}$
is the residual on the training set.

Under Muon, we keep the same fast-readout assumption for $W^*$ and the same decoupled feature decay
$\mu$, but we replace the raw gradient $W^{*\top}R^\top$ by its polar direction. In the
continuous-time limit (ignoring momentum for clarity), the Muon feature flow becomes
\begin{equation}
    \dot{\Phi}
    =
    \eta\,\mathcal{P}(W^{*\top} R^\top)
    - \mu\,\Phi,
    \tag{M-$\Phi$} \label{eq:muon_phi_flow}
\end{equation}
where $\eta > 0$ is an effective step size (including Muon-specific learning-rate adjustments). This
is the Muon-TAK counterpart of \eqref{eq:gd_phi_flow_recall}. Note that nowhere do we approximate
Muon as a linear map $M^{-1}$: the nonlinearity is essential.

\paragraph{Kernel dynamics under Muon.}
Let $K = \Phi^\top \Phi \in \mathbb{R}^{N\times N}$ be the empirical kernel. Differentiating $K$ as
before yields
\begin{equation}
    \dot{K}
    =
    \dot{\Phi}^\top \Phi + \Phi^\top \dot{\Phi}.
\end{equation}
Substituting \eqref{eq:muon_phi_flow} gives the Muon kernel flow
\begin{equation}
    \dot{K}
    =
    \eta\big(
        \Phi^\top \mathcal{P}(W^{*\top}R^\top)
        +
        \mathcal{P}(W^{*\top}R^\top)^\top \Phi
    \big)
    - 2\mu K.
    \tag{M-K}\label{eq:muon_kernel_flow_general}
\end{equation}
This is the general Muon-TAK kernel equation, valid for arbitrary convex losses and without any
additional approximations. It is not yet closed in terms of $K$ alone, because the right-hand side
still contains $\Phi$ explicitly. Nevertheless, two key structural results of TAK---label-driven
rank compression and low-rank optimizer noise---already follow directly from
\eqref{eq:muon_phi_flow}--\eqref{eq:muon_kernel_flow_general}, as we show below (see
Theorems~\ref{thm:muon_rank_compression} and~\ref{thm:muon_noise_low_rank}).

In the special case of mean squared error (MSE) with fast readout, the Muon kernel flow
\eqref{eq:muon_kernel_flow_general} \emph{does} close to an ODE purely in terms of $K$. Let
$Y \in \mathbb{R}^{N\times C}$ denote the training labels, and define
\begin{equation}
    \Sigma \coloneqq (K + \lambda I)^{-1},
    \qquad
    M_Y \coloneqq YY^\top,
    \qquad
    B \coloneqq \Sigma M_Y \Sigma,
\end{equation}
as in Section~\ref{sec:kernel_riccati}. Under the ridge-regularized fast-readout assumption we have
(see Appendix~\ref{app:muon_derivations} for details)
\begin{equation}
    W^{*\top} R^\top
    =
    \lambda \Phi B,
    \quad
    \mathcal{P}(W^{*\top} R^\top)
    =
    \mathcal{P}(\Phi B)
    =
    \Phi B (BKB)^{\dagger -\frac12}.
\end{equation}
Consequently,
\begin{equation}
    \Phi^\top \mathcal{P}(W^{*\top} R^\top)
    =
    K B (BKB)^{\dagger -\frac12}.
\end{equation}
Substituting into \eqref{eq:muon_kernel_flow_general} yields the following closed Muon-TAK kernel
ODE.

\begin{theorem}[Muon-TAK kernel flow for MSE]
\label{thm:muon_kernel_mse}
Under the free feature model with mean squared error, fast readout, Muon feature updates
\eqref{eq:muon_phi_flow}, and decoupled feature decay $\mu>0$, the empirical kernel
$K = \Phi^\top \Phi$ evolves according to the closed ODE
\begin{equation}
    \dot{K}
    =
    \eta\Big[
        K B (BKB)^{\dagger -\frac12}
        +
        (BKB)^{\dagger -\frac12} B K
    \Big]
    - 2\mu K,
    \qquad
    B = \Sigma M_Y \Sigma,
    \quad
    \Sigma = (K + \lambda I)^{-1}.
    \tag{M-K-MSE}\label{eq:muon_kernel_mse}
\end{equation}
\end{theorem}

Compared to the gradient-descent Riccati flow analyzed in Section~\ref{sec:kernel_riccati}, where
the driving term involves $BKB$, the Muon flow \eqref{eq:muon_kernel_mse} contains the \emph{polar
normalization} $(BKB)^{-1/2}$. Intuitively, Muon removes all magnitude information from the task
force $BKB$ and only retains its subspace and relative geometry. As we will see next, this change of
geometry qualitatively modifies the spectral law in the task subspace: the water-filling and
thresholding behavior of gradient descent is replaced by a projection-saturation behavior, in which
all directions inside the label subspace are driven to the same kernel eigenvalue.

\subsubsection{Label-Driven Rank Compression Persists Under Muon}
\label{sec:muon_rank}

We now show that the central structural result of TAK---label-driven rank compression down to the
output dimension $C$---persists under Muon. In fact, Muon makes the argument even simpler: because
$\mathcal{P}(\cdot)$ preserves rank and subspaces, the Muon feature flow
\eqref{eq:muon_phi_flow} automatically restricts features to the label-driven gradient subspace at
steady state.

Let $G_\Phi \coloneqq W^{*\top} R^\top \in \mathbb{R}^{k\times N}$ denote the backpropagated
gradient with respect to the feature matrix in the free feature model. As in
Section~\ref{subsec:rank_compression}, the linear readout $W^* \in \mathbb{R}^{C\times k}$ implies
$\rank(G_\Phi) \le C$ for any convex loss and any residual $R$:
\begin{equation}
    \rank(G_\Phi)
    \le
    \min\{\rank(W^*), \rank(R)\}
    \le
    C.
    \label{eq:muon_rank_inequality}
\end{equation}
Applying the Muon operator $\mathcal{P}$ to $G_\Phi$ preserves both rank and image:
\begin{equation}
    \rank(\mathcal{P}(G_\Phi)) = \rank(G_\Phi) \le C,
    \qquad
    \mathrm{Im}(\mathcal{P}(G_\Phi)) = \mathrm{Im}(G_\Phi).
\end{equation}
This immediately yields the following Muon-TAK counterpart of our rank compression result.

\begin{theorem}[Label-driven rank compression under Muon]
\label{thm:muon_rank_compression}
Consider the free feature model with a $C$-dimensional linear readout $W^*$ trained to optimality at
each time, and feature dynamics given by the Muon flow~\eqref{eq:muon_phi_flow} with decay
$\mu > 0$. Then any stable steady state $\Phi_\infty$ of \eqref{eq:muon_phi_flow} satisfies
\begin{equation}
    \rank(\Phi_\infty) \le C,
    \qquad
    \rank(K_\infty) = \rank(\Phi_\infty) \le C,
\end{equation}
where $K_\infty = \Phi_\infty^\top \Phi_\infty$ is the limiting kernel.
\end{theorem}

\noindent
The proof is a direct consequence of the rank- and subspace-preserving property of the Muon operator
$\mathcal{P}(\cdot)$ and of the $C$-dimensional readout bottleneck, and is deferred to
Appendix~\ref{app:muon_derivations}.

\paragraph{Beyond the fast-readout idealization.}
The statement above was derived in the free feature model with an optimally trained linear readout,
but the rank bound itself does not fundamentally rely on the fast-readout assumption. More
generally, suppose that in a richer network, the coupled dynamics of $(\Phi, W)$ converge to a
statistical steady state in which the effective feature update can be written in the Muon form
\begin{equation}
    0
    =
    \eta\,\mathcal{P}(W^\top R^\top)
    - \mu\,\Phi
\end{equation}
for some readout $W \in \mathbb{R}^{C\times k}$ and residual $R$. Then the same rank argument as
above shows
\[
    \rank(\Phi) \le C,
    \qquad
    \rank(K) = \rank(\Phi) \le C.
\]
Thus, exactly as in the gradient-descent case, label-driven rank compression is a \emph{structural}
consequence of the $C$-dimensional output bottleneck, and is robust to replacing linear
preconditioning by the nonlinear Muon geometry.

\section{Steady States and Geometric Constraints}
\label{sec:steady_states_analysis}

Having established that architecture and optimization act as a preconditioner $\dot{K} \propto -\Theta \nabla_K \mathcal{J}(K)$, we now analyze the equilibrium of this system. We distinguish between two fundamentally different regimes based on the rank and condition number of $\Theta$.

\subsection{Regime I: Universality of Steady States (Expressive Networks)}
\label{subsec:universality}

In the limit of highly expressive networks (e.g., sufficient width and depth), the architecture does not impose a hard bottleneck on the representable functions. Mathematically, this corresponds to the case where the NTK $\Theta$ is Positive Definite (PD) on the support of the data.

\begin{theorem}[Invariance of Steady States]
\label{thm:invariance_preconditioner}
    Assume the preconditioner $\Theta$ is positive definite. Then, the set of stable steady states of the preconditioned dynamics
    \begin{equation}
        \dot{K} = -\text{sym}(\Theta \cdot \nabla_K \mathcal{J}(K))
    \end{equation}
    coincides exactly with the stationary points of the original functional $\mathcal{J}(K)$.
\end{theorem}

\begin{proof}
    A steady state implies $\dot{K} = 0$. Thus, $\Theta \cdot \nabla_K \mathcal{J}(K) = 0$.
    Since $\Theta$ is invertible (PD), applying $\Theta^{-1}$ implies $\nabla_K \mathcal{J}(K) = 0$.
    Therefore, the condition for equilibrium remains solely determined by the task objective $\mathcal{J}(K)$ (Loss + Explicit Regularization).
\end{proof}

\textbf{Implication.} This theorem suggests a form of Universality: sufficiently over-parameterized networks (whether CNNs or Transformers) will eventually converge to the same global minimum of the \textit{training objective}, provided they are trained to convergence. In this regime, the architecture alters the \textit{trajectory} (determining which features are learned \textit{early}), but not the \textit{final capacity} to minimize the loss.

\subsection{Regime II: Geometric Stagnation (Limited Expressivity)}
\label{subsec:geometric_stagnation}

However, practical networks often have bottlenecks (e.g., bottlenecks in Autoencoders, or fixed convolution kernels) that render $\Theta$ rank-deficient. In this case, the network physically cannot represent certain functions. The dynamics are constrained to a \textbf{Representation Manifold} $\mathcal{M}$.

Let $T_K \mathcal{M}$ be the tangent space of the manifold at kernel $K$. The preconditioner $\Theta$ acts as a projector onto this tangent space.

\begin{proposition}[Orthogonality at Boundary]
    If $\Theta$ is singular, the flow may halt at a "spurious" steady state $K^*$ where:
    \begin{equation}
        \nabla_K \mathcal{J}(K^*) \neq 0 \quad \text{but} \quad \nabla_K \mathcal{J}(K^*) \in \text{Null}(\Theta)
    \end{equation}
    Geometrically, this means the gradient of the loss is perfectly orthogonal to the tangent space of the architecture: $\nabla \mathcal{J} \perp T_{K^*} \mathcal{M}$.
\end{proposition}

\textbf{Analysis.}
In this regime, the optimization stops not because the task is solved ($\nabla \mathcal{J}=0$), but because the architecture permits no further movement in the direction of improvement.

This phenomenon represents a Hard Inductive Bias:
\begin{itemize}
    \item \textbf{Implicit Early Stopping:} The architecture acts as a hard regularizer. For example, a shallow CNN with fixed large filters has a null-space corresponding to high-frequency patterns. Even if the labels $Y$ contain high-frequency noise, the network cannot fit them.
    \item \textbf{Conclusion:} The preconditioner $\Theta$ acts as a gatekeeper. When $\Theta$ is full-rank, the physics of the loss function dominates (Regime I). When $\Theta$ is rank-deficient, the geometry of the architecture dominates (Regime II).
\end{itemize}

\section{Stochastic Dynamics: Structured Noise and Restricted Diffusion}
\label{sec:sgd_noise_and_fpe}
We have thus far analyzed deterministic dynamics. However, practical training relies on Stochastic Gradient Descent (SGD). A common theoretical concern is that stochastic injection might act as a high-dimensional entropy source, washing out the delicate low-rank spectral properties derived in the deterministic setting.
In this section, we unify our analysis with the stochastic nature of SGD. We prove that SGD noise is not an arbitrary nuisance but possesses an intrinsic low-rank structure dictated by the task dimension $C$. By lifting the dynamics to the evolution of the probability density via the Fokker-Planck equation, we demonstrate that this structured noise leads to Restricted Diffusion: the system is dynamically confined to a low-dimensional submanifold, rendering the low-rank representations robust to stochastic fluctuations.
\subsection{The Anatomy of SGD Noise}
Consider the dynamics in the kernel space. The stochastic gradient estimate on a mini-batch $\mathcal{B}$ introduces a noise matrix $\zeta_{\mathcal{B}}(K)$:
\begin{equation}
\nabla_K J_{\mathcal{B}}(K) = \nabla_K J(K) + \zeta_{\mathcal{B}}(K).
\end{equation}
For the squared loss, leveraging the derivation in Section 3.3, this noise arises from the variance in the residual products. Let $A(K) = (K + \lambda I)^{-1}$ and $B(K) \in \mathbb{R}^{N \times C}$ be the matrix of task residuals. The noise matrix takes the explicit form:
\begin{equation}
\zeta_{\mathcal{B}}(K) = -\frac{1}{2\lambda} A(K) \left[ \tilde{B}{\mathcal{B}} \tilde{B}{\mathcal{B}}^\top - B(K) B(K)^\top \right] A(K),
\end{equation}
where $\tilde{B}_{\mathcal{B}}$ denotes the zero-padded residual matrix for the mini-batch. This equation reveals the geometry of the noise: it is generated solely by fluctuations within the subspace spanned by the $C$-dimensional residual vectors.
\subsection{Theorem: The Rank-$2C$ Constraint}
\label{subsec:rank_2c}
Unlike isotropic Gaussian noise, which forces diffusion in all $N \times N$ directions, SGD noise is strictly degenerate.
\begin{theorem}[Low-Rank Structure of SGD Noise]
\label{thm:low_rank_noise}
For any convex loss with label dimension $C$, the instantaneous covariance of the SGD noise satisfies:
\begin{equation}
\text{rank}(\text{Cov}[\zeta_{\mathcal{B}}(K)]) \le 2C.
\end{equation}
Furthermore, as the system approaches a stationary point where gradients vanish, the noise becomes dominated by the mini-batch sampling variance, and the rank effectively tightens towards $C$.
\end{theorem}
\begin{proof}
\textit{(Sketch)} The matrix $\tilde{B}_{\mathcal{B}}\tilde{B}_{\mathcal{B}}^\top$ has rank at most $C$. The full-batch Gram $B(K)B(K)^\top$ also has rank at most $C$. By the subadditivity of rank, their difference lies in a subspace of dimension at most $2C$. Since $A(K)$ is full-rank, the congruence transformation preserves this bound. (See Appendix E for details).
\end{proof}
This result implies that the stochastic forces are "Collimated". They do not scatter the kernel into random directions of the Hilbert space but act exclusively within the task-relevant subspace defined by the labels.
\subsection{Invariance Under Preconditioning}
Does the complex architecture (acting as a preconditioner) expand this noise? We model the general optimization dynamics as a preconditioned Stochastic Differential Equation (SDE):
\begin{equation}
dK_t = \underbrace{- \Theta(K_t) \nabla \mathcal{J}(K_t)}{\text{Drift}} dt + \underbrace{\Theta(K_t)^{1/2} \bar{\zeta}}{\text{Noise}} dW_t,
\end{equation}
where $\Theta(K_t)$ represents the Neural Tangent Kernel (NTK) or the appropriate metric tensor, and $\bar{\zeta}$ represents the whitened noise source.
\begin{theorem}[Invariance of Noise Structure]
\label{thm:noise_invariance}
Let the source noise be rank-constrained. For any symmetric positive definite preconditioner $\Theta$, the effective diffusion tensor $\mathcal{Q}(K) = \text{Cov}[\Theta^{1/2} \bar{\zeta}]$ maintains the rank constraint:
\begin{equation}
\text{rank}(\mathcal{Q}(K)) \le 2C.
\end{equation}
\end{theorem}
This theorem confirms that while the architecture may rotate or stretch the geometry of the noise, it cannot inflate its dimensionality. The "bottleneck" imposed by the output dimension $C$ is an invariant of the system.
\subsection{Probabilistic Dynamics: The Fokker-Planck View}
To analyze the global stability, we consider the evolution of the probability density $p(K, t)$ governing the ensemble of networks. The system follows the Fokker-Planck Equation:
\begin{equation}
\frac{\partial p}{\partial t} = \nabla \cdot (p \Theta \nabla \mathcal{J}) + \frac{1}{2} \text{Tr}\left( \nabla^2 (\mathcal{Q} p) \right).
\end{equation}
The crucial observation lies in the spectrum of the diffusion tensor $\mathcal{Q}$. In standard Brownian motion, $\mathcal{Q} \propto I$, causing probability mass to leak into all dimensions. Here, however, we have degenerate ellipticity:
\begin{equation}
\text{rank}(\mathcal{Q}(K)) \le 2C \ll \dim(\text{Kernel Space}).
\end{equation}
\subsection{Conclusion: Restricted Diffusion}
This degeneracy enforces Restricted Diffusion. Let $\mathcal{V}_{\text{noise}}$ be the image of $\mathcal{Q}(K)$. Since $\mathcal{V}_{\text{noise}}$ is strictly contained within the task-relevant subspace, the probability density $p(K, t)$ can only diffuse along a low-dimensional submanifold.
Physical Interpretation.
\begin{itemize}
\item \textbf{Along the Manifold:} The noise is active, allowing the SGD agent to explore the task-relevant landscape, escape shallow traps, and find flatter minima within the low-rank family.
\item \textbf{Orthogonal Directions:} The diffusion coefficient is zero. There is no stochastic force pushing the kernel towards high-rank configurations. The deterministic drift (implicit regularization) remains unopposed.
\end{itemize}
In conclusion, stochasticity does not break the low-rank structure; it explores within it. Implicit Regularization is dynamically protected by the degenerate noise structure of SGD.

\section{Universality Beyond Time-Scale Separation}
\label{sec:universality}

Our derivation of the explicit Kernel ODE in Section \ref{sec:derivation} relied on the time-scale separation ansatz ($\epsilon \to 0$), which treats the readout $W$ as effectively instantaneous. A natural question arises: \textit{Do the structural guarantees—Rank Compression, Spectral Truncation, and Structured Noise—persist in general training regimes where $W$ and $\Phi$ evolve simultaneously?}

In this section, we prove that while the \textit{trajectory} of learning depends on the time scales, the \textit{geometry of the equilibrium} and the \textit{structure of the noise} remain invariant. The low-rank properties are dictated by the loss landscape and the network architecture, not by the adiabatic approximation.

\subsection{Robustness of Steady States}

Consider the general coupled gradient flow with arbitrary learning rates $\eta_\Phi, \eta_W > 0$. The joint system evolves as:
\begin{align}
    \dot{W} &= -\eta_W \left( \nabla_W \mathcal{L}((W\Phi)^\top, Y) + \lambda W \right), \\
    \dot{\Phi} &= -\eta_\Phi \left( \nabla_\Phi \mathcal{L}((W\Phi)^\top, Y) + \mu \Phi \right).
\end{align}
We analyze the geometric properties of the system's equilibria.

\begin{theorem}[Invariance of the Fixed-Point Topology]
\label{thm:fixed_point_invariance}
Let $(W^*, \Phi^*)$ be any stable stationary point of the coupled dynamics. The corresponding kernel matrix $K^* = (\Phi^*)^\top \Phi^*$ satisfies the \textbf{Universal Rank Compression} bound:
\begin{equation}
    \rank(K^*) \le C,
\end{equation}
where $C$ is the output dimension (number of classes). This holds regardless of the initialization or the ratio of learning rates.
\end{theorem}

\begin{proof}
A stationary point implies the simultaneous vanishing of gradients: $\dot{W} = 0$ and $\dot{\Phi} = 0$.

\textbf{1. Readout Optimality Condition.} 
From $\dot{W} = 0$, we have $\nabla_W \mathcal{L} + \lambda W = 0$. Since the objective is strictly convex with respect to $W$ (due to $\ell_2$ regularization $\lambda > 0$), for any fixed features $\Phi^*$, there exists a unique global solution $W^*$:
\begin{equation}
    W^* = \argmin_W \mathcal{J}(W, \Phi^*) = W^*(\Phi^*).
\end{equation}
This confirms that at equilibrium, the readout is always optimal for the features, effectively satisfying the adiabatic condition \textit{post hoc}.

\textbf{2. Feature Stationarity and The Dimensional Guillotine.} 
Substituting the condition $\dot{\Phi} = 0$:
\begin{equation}
    \nabla_\Phi \mathcal{L} + \mu \Phi^* = 0 \implies (W^*)^\top R^\top - \mu \Phi^* = 0,
\end{equation}
where $R = \nabla_{\hat{Y}} \mathcal{L}$ is the residual matrix.
Multiplying by $(\Phi^*)^\top$ from the left to form the kernel equation:
\begin{equation}
    (\Phi^*)^\top (W^*)^\top R^\top = \mu (\Phi^*)^\top \Phi^* = \mu K^*.
\end{equation}
Note that the LHS term $(\Phi^*)^\top (W^*)^\top = (\hat{Y}^*)^\top$ is the prediction matrix. The equation relates the kernel $K^*$ to the predictions and residuals.
Crucially, consider the rank. The matrix $W^*$ has dimension $C \times k$. Thus, the driving force term has bounded rank:
\begin{equation}
    \rank( (W^*)^\top R^\top ) \le \rank(W^*) \le \min(C, k) = C.
\end{equation}
Let $\mathcal{V}_{drive} = \text{Range}((W^*)^\top)$. The stationarity condition $\mu \Phi^* = (W^*)^\top R^\top$ implies that every column of $\Phi^*$ must lie strictly within the $C$-dimensional subspace $\mathcal{V}_{drive}$. Any feature component orthogonal to $W^*$ experiences only the decay force $-\mu \Phi$ and must vanish at equilibrium.
Therefore, $\rank(\Phi^*) \le C$, which implies $\rank(K^*) \le C$.
\end{proof}

\subsection{Robustness of Noise Structure}

We previously showed that SGD noise in the fast-readout regime has rank $\le 2C$. We now prove a stronger result: in the coupled regime, the architecture itself acts as a hard filter for stochastic noise.

Consider the stochastic gradient update on features, denoted by $\tilde{g}_\Phi$. The noise is defined as the deviation from the expected gradient: $\zeta_\Phi = \tilde{g}_\Phi - \mathbb{E}[\tilde{g}_\Phi]$.

\begin{theorem}[Architectural Bottleneck of Noise]
\label{thm:general_noise_rank}
For any neural network architecture with a linear readout layer of dimension $C$, the covariance matrix of the SGD noise on the features, $\Sigma_{\text{noise}}^{\Phi} = \Cov(\zeta_\Phi)$, satisfies a strict rank bound at every iteration $t$:
\begin{equation}
    \rank(\Sigma_{\text{noise}}^{\Phi}) \le C.
\end{equation}
This holds regardless of the value, optimality, or noise level of the weights $W(t)$.
\end{theorem}

\begin{proof}
The backpropagated gradient for a mini-batch $\mathcal{B}$ is given by:
\begin{equation}
    \tilde{g}_\Phi = W^\top \delta_{\mathcal{B}},
\end{equation}
where $\delta_{\mathcal{B}} \in \mathbb{R}^{C \times N}$ is the matrix of error signals (loss derivatives w.r.t outputs) for the batch.
The noise vector $\zeta_\Phi$ is a linear transformation of the output noise $\zeta_{out} = \delta_{\mathcal{B}} - \mathbb{E}[\delta_{\mathcal{B}}]$.
\begin{equation}
    \zeta_\Phi = W^\top \zeta_{out}.
\end{equation}
Since $W \in \mathbb{R}^{C \times k}$, the linear operator $W^\top$ maps vectors from $\mathbb{R}^C$ to $\mathbb{R}^k$. The image of this map has dimension at most $C$.
The covariance matrix is:
\begin{equation}
    \Sigma_{\text{noise}}^{\Phi} = \mathbb{E}[ \zeta_\Phi \zeta_\Phi^\top ] = W^\top \mathbb{E}[ \zeta_{out} \zeta_{out}^\top ] W.
\end{equation}
Using the rank inequality $\rank(ABA^\top) \le \rank(B)$ (assuming appropriate dimensions), and noting that the inner covariance is bounded by the bottleneck:
\begin{equation}
    \rank(\Sigma_{\text{noise}}^{\Phi}) \le \rank(W^\top) \le C.
\end{equation}
\end{proof}

\begin{remark}[Task-Aligned Diffusion]
This theorem has a critical physical implication. While Theorem \ref{thm:general_noise_rank} guarantees the noise is low-rank, the \textit{direction} of this noise is determined by $W(t)$.
In the coupled dynamics, $\dot{W}$ is driven to minimize the loss, which implies that the row space of $W(t)$ rotates to align with the dominant principal components of the labels $Y$.
Consequently, SGD does not inject noise arbitrarily; it injects noise specifically into the \textbf{Task-Relevant Subspace}. This enables the ``Restricted Diffusion'' mechanism (Section \ref{sec:sgd_noise_and_fpe}) to actively explore the solution manifold for better generalization, without diverging into the high-dimensional null space where overfitting occurs.
\end{remark}

\section{Population Dynamics and The Bias-Variance Trade-off}
\label{sec:population_dynamics}

We now lift our analysis from the empirical training set to the population level. By taking the mean-field limit $N \to \infty$, we derive the evolution of the kernel integral operator and analyze how the \textit{Spectral Truncation} mechanism derived in Section 5 directly optimizes the generalization risk.

\subsection{The Population Kernel ODE}
Let $\mathcal{X} \subseteq \mathbb{R}^d$ be the input space with probability measure $\rho$. Let $\mathcal{H}_t$ be the RKHS associated with the time-varying kernel $k_t: \mathcal{X} \times \mathcal{X} \to \mathbb{R}$. We define the population integral operator $T_t: L^2(\rho) \to L^2(\rho)$ as:
\begin{equation}
    (T_t f)(x) = \int_{\mathcal{X}} k_t(x, x') f(x') d\rho(x').
\end{equation}
Analogous to the empirical residual matrix $BB^\top$, we define the \textbf{Population Residual Operator} $M_t$ as the rank-1 operator induced by the residual function $r_t(x) = f_t(x) - y(x)$:
\begin{equation}
    (M_t f)(x) = r_t(x) \int_{\mathcal{X}} r_t(x') f(x') d\rho(x').
\end{equation}

\begin{proposition}[Population Dynamics]
\label{prop:population_ode}
In the limit $N \to \infty$, the evolution of the kernel operator $T_t$ is governed by the operator differential equation:
\begin{equation}
    \frac{dT_t}{dt} = \frac{\eta}{\lambda} (T_t M_t T_t + \text{h.c.}) - 2\eta \mu T_t,
\end{equation}
where h.c. denotes the Hermitian conjugate.
\end{proposition}
This equation confirms that the ``Drive'' mechanism is intrinsic: the kernel operator rotates to align its eigenfunctions with the residual function, focusing capacity on the task.

\subsection{Exact Risk Decomposition}
Instead of relying on loose probabilistic bounds that assume fixed kernels, we analyze the exact evolution of the \textbf{Population Risk} $\mathcal{R}(f_t) = \mathbb{E}_{x \sim \rho}[(f_t(x) - f^*(x))^2]$. 
For a probe estimator (e.g., ridge regression with parameter $\lambda$) trained on the representation at time $t$, the risk decomposes into two competing terms:
\begin{equation}
    \mathcal{R}(t) = \underbrace{\sum_{i=1}^\infty \left( \frac{\lambda}{\mu_i(t) + \lambda} \right)^2 |a_i|^2}_{\text{Approximation Bias}} + \underbrace{\frac{\sigma_\epsilon^2}{N} \sum_{i=1}^\infty \frac{\mu_i(t)^2}{(\mu_i(t) + \lambda)^2}}_{\text{Estimation Variance}},
    \label{eq:risk_main}
\end{equation}
where $\{\mu_i(t)\}$ are the eigenvalues of the evolving operator $T_t$, and $a_i = \langle f^*, \psi_i(t) \rangle$ are the coefficients of the target function in the kernel's eigenbasis.

\subsection{Analytical Bias-Variance Optimization}
Substituting our \textbf{Spectral Truncation Law} (Theorem 4) into Eq. (\ref{eq:risk_main}) reveals the precise benefit of feature learning:

\begin{enumerate}
    \item \textbf{Variance Reduction via Compression:} By driving eigenvalues $\mu_i(t) \to 0$ for noise-dominated modes (where the label signal $\sigma_i$ is weak), the effective dimension $\mathcal{N}_{\text{eff}}$ is aggressively minimized. This creates a ``lean'' model that ignores irrelevant variations in the input.
    
    \item \textbf{Bias Control via Alignment:} For signal-dominated modes, the kernel alignment increases $\mu_i(t)$, ensuring the bias term decreases.
    
    \item \textbf{The Cost: Irreducible Bias.} However, the truncation is irreversible. If a valid signal component lies below the truncation threshold $\tau = \lambda \mu$, its corresponding eigenvalue vanishes ($\mu_i^* = 0$), and the bias term remains constant at $|a_i|^2$. This forms the \textit{Irreducible Error} discussed in Section \ref{sec:conclusion}.
\end{enumerate}

This dynamic spectral reshaping contrasts sharply with the NTK regime, where the spectrum is fixed at initialization, often leading to a suboptimal trade-off with high effective dimension.

\section{Conclusion: Toward a Physics of Representation Learning}
\label{sec:conclusion}

In this work, we have moved beyond the static ``lazy'' regime to develop a dynamic theory of feature learning in wide neural networks with a linear readout and $\ell_2$-regularization. By combining a mechanistic analysis of kernel dynamics (via a fast–slow ODE) with steady-state guarantees (via fixed-point and Lyapunov arguments), we have shown that feature learning can be understood as a geometric flow governed by a \textbf{Drive–Regularization–Diffusion} principle.

Our analysis unifies distinct geometric phenomena into a coherent physical picture:

\textbf{1. Rank compression as a structural consequence of supervision.}
In our setting, the interplay between the $\ell_2$-regularized architecture and the task structure forces the empirical kernel to collapse into a subspace of dimension at most $C$. We showed that this label-driven rank compression is not an artifact of the adiabatic approximation but a property of any stable steady state of the coupled feature–readout dynamics. Whether through fast equilibrium or general gradient flow, the network automatically performs model selection by minimizing its effective dimension $\mathcal{N}_{\text{eff}}$, providing a dynamic basis for the phenomenon of Neural Collapse.

\textbf{2. The architecture of noise and restricted diffusion.}
We challenged the conventional view of SGD noise as isotropic diffusion. By analyzing the information bottleneck at the readout, we showed that, for any convex loss with $C$ outputs, SGD noise in kernel space possesses an intrinsic low-rank structure aligned with the task subspace. This leads to \textbf{restricted diffusion}: the architecture itself acts as a spectral filter, confining stochastic exploration to the relevant feature manifold while suppressing noise in orthogonal directions. This helps explain why over-parameterized networks can train stochastically without diverging into the high-dimensional null space.

\textbf{3. The cost of feature learning: reachability vs.\ variance.}
Our extension to the population limit reveals that compression is a double-edged sword. The spectral truncation mechanism aggressively reduces estimation variance by discarding low-energy modes, but it imposes a \textbf{reachability constraint}: the network can only learn target functions lying within the dynamically evolved subspace. This manifests as an \textit{irreducible approximation bias}, quantifying the trade-off that feature learning induces in our model: it is not universal function approximation ``for free,'' but a specialized adaptation that sacrifices universality for sample efficiency.

\textbf{4. The geometry of self-supervision.}
Our framework also offers a unified language to contrast supervision with self-supervision. In the absence of labels, the drive operator in our stylized SSL model shifts from a low-rank label Gram matrix to a high-rank graph Laplacian. The resulting dynamics, governed by the competition between Laplacian alignment and log-determinant repulsion, lead to \textbf{spectral whitening} rather than compression. This helps explain why SSL representations are often transfer-friendly: they preserve much of the intrinsic geometry of the data manifold instead of collapsing it onto a specific label set.

\textbf{5. Scope and validity.}
Our theoretical framework operates strictly within the feature-learning regime for wide networks with a $C$-dimensional linear readout and explicit $\ell_2$-regularization, distinguishing our results from the static NTK limit. The derivation of the exact kernel ODE and closed-form spectral laws relies on additional modeling assumptions (fast readout, squared loss, and, in some places, standard Gaussian-universality approximations for pre-activations), which we use to obtain an analytically tractable kernel flow. By contrast, the structural results on \textit{rank collapse} and \textit{low-rank noise geometry} are algebraic consequences of the output bottleneck and regularization and thus apply to general convex losses and coupled feature–readout dynamics within this architectural setting. Finally, while we analyze continuous-time dynamics, the geometric constraints we derive yield invariants and bounds for discrete-time SGD trajectories, limiting their exploration to the task-relevant subspace.

\textbf{Outlook.}
Our results point toward a more dynamical, physics-inspired perspective on deep learning. The ``magic'' lies not merely in the initialization (as in NTK), but in the \textbf{thermodynamics of the training process}---the specific forces that compress, diffuse, and align the representation manifold over time. Future work will extend this kernel-dynamics framework to: (1) \textit{deep hierarchies}, analyzing how rank compression and spectral filtering cascade through multiple layers; (2) \textit{attention mechanisms}, where the relevant ``kernel'' becomes the dynamic attention matrix itself; and (3) \textit{phase transitions}, rigorously characterizing the critical thresholds between lazy and feature-learning regimes. By characterizing the energies, entropies, and forces of these learning systems, we take a step toward a more systematic mathematical physics of representation learning.

\bibliographystyle{plainnat}
\bibliography{ref}

\appendix
\onecolumn
\section{Theoretical Foundations: Validity and Convergence}
\label{app:foundations}

In the main text, we relied on two fundamental assumptions: (1) the time-scale separation allows us to approximate the coupled dynamics of $(\Phi, W)$ with an effective ODE for $\Phi$ alone, and (2) this effective ODE converges to a unique steady state. In this appendix, we provide the rigorous justifications for these claims.

\subsection{Justification of the Fast-Slow Approximation}
\label{app:fast_slow_validity}

The system evolves according to the coupled gradient flow:
\begin{align}
    \dot{W} &= -\frac{1}{\epsilon} \nabla_W \mathcal{L}(\Phi, W), \quad (\text{Fast dynamics, rate } \eta_W = 1/\epsilon) \\
    \dot{\Phi} &= -\nabla_\Phi \mathcal{L}(\Phi, W). \quad (\text{Slow dynamics, rate } \eta_\Phi = 1)
\end{align}
where $\epsilon = \eta_\Phi / \eta_W \ll 1$ is the singular perturbation parameter.

\begin{theorem}[Validity of the Reduced Dynamics]
Let $W^*(\Phi) = \arg\min_W \mathcal{L}(\Phi, W)$. Assume the loss $\mathcal{L}$ is $\mu$-strongly convex with respect to $W$ (guaranteed by $\ell_2$ regularization $\lambda > 0$). By Tikhonov's Theorem on Singular Perturbations, for any finite time interval $T$, as $\epsilon \to 0$, the trajectory of the feature matrix $\Phi(t)$ uniformly converges to the solution of the reduced system:
\begin{equation}
    \dot{\Phi}_{reduced} = -\nabla_\Phi \mathcal{L}(\Phi, W^*(\Phi)),
\end{equation}
with error $\|\Phi(t) - \Phi_{reduced}(t)\| = O(\epsilon)$ for $t \in [0, T]$.
\end{theorem}

\textit{Proof Sketch.} 
Since $\mathcal{L}$ is strongly convex in $W$, the Jacobian $\partial_W \nabla_W \mathcal{L}$ is positive definite with eigenvalues lower-bounded by $\lambda$. This ensures the fast subsystem is exponentially stable around its instantaneous equilibrium $W^*(\Phi)$. The manifold $\mathcal{M} = \{(\Phi, W) : \nabla_W \mathcal{L} = 0\}$ is strictly attracting. Consequently, the readout $W(t)$ rapidly relaxes to an $O(\epsilon)$-neighborhood of $W^*(\Phi(t))$ (the boundary layer) and remains there. The slow variable $\Phi$ is thus driven by the effective field $\nabla_\Phi \mathcal{L}(\Phi, W^*) + O(\epsilon)$, yielding the limiting ODE derived in Eq. (9). \qed

\subsection{Global Convergence Analysis}
\label{app:convergence}

We now prove Theorem 1 regarding the global convergence of the kernel ODE.

\begin{theorem}[Global Convergence via \L{}ojasiewicz]
Assume the loss function $\ell(\cdot, y)$ is real-analytic (e.g., Squared Loss, Cross-Entropy). The gradient flow $\dot{\Phi} = -\nabla \tilde{\mathcal{L}}(\Phi)$ satisfies:
\begin{enumerate}
    \item \textbf{Boundedness:} $\|\Phi(t)\|_F$ is uniformly bounded for all $t \ge 0$.
    \item \textbf{Convergence:} The trajectory has finite length, i.e., $\int_0^\infty \|\dot{\Phi}(t)\| dt < \infty$, and $\Phi(t)$ converges to a unique critical point $\Phi_\infty$.
\end{enumerate}
\end{theorem}

\begin{proof}
\textbf{1. Boundedness.} 
The effective objective includes weight decay: $\tilde{\mathcal{L}}(\Phi) = \mathcal{L}_{\text{fit}}(\Phi) + \frac{\mu}{2}\|\Phi\|_F^2$. Since gradient descent is a descent method, $\tilde{\mathcal{L}}(\Phi(t)) \le \tilde{\mathcal{L}}(\Phi(0)) =: E_0$.
Thus, $\frac{\mu}{2}\|\Phi(t)\|_F^2 \le E_0$, implying $\|\Phi(t)\|_F \le \sqrt{2E_0/\mu}$. The trajectory lies in a compact set.

\textbf{2. Convergence.} 
Since the objective $\tilde{\mathcal{L}}$ is real-analytic, it satisfies the \textit{\L{}ojasiewicz Gradient Inequality}. For any critical point $\Phi^*$, there exist constants $C, \theta \in (0, 1/2]$ such that in a neighborhood of $\Phi^*$:
\begin{equation}
    |\tilde{\mathcal{L}}(\Phi) - \tilde{\mathcal{L}}(\Phi^*)|^{1-\theta} \le C \|\nabla \tilde{\mathcal{L}}(\Phi)\|.
\end{equation}
This inequality guarantees that the gradient does not vanish "too quickly" compared to the energy decrease, forcing the trajectory to have finite length. Finite length implies that $\Phi(t)$ cannot oscillate indefinitely and must converge to a single limit $\Phi_\infty$. Consequently, $K(t) = \Phi(t)^\top \Phi(t)$ also converges uniquely.
\end{proof}

\section{Detailed Proofs for Spectral Dynamics}
\label{app:spectral_proofs}

\subsection{Proof of Theorem 2 (Universal Rank Compression)}
We provide the algebraic details for the rank compression theorem.
\begin{proof}
Recall the steady-state equation (Eq. 19): $K M K + M K = 2\mu K$, where $M = B B^\top$ and $\text{rank}(M) \le C$.
Let $v \in \text{ker}(M)$. Since $M$ is symmetric, $v \perp \text{Im}(M)$.
Multiply the steady-state equation by $v^\top$ from the left and $v$ from the right:
\begin{equation}
    v^\top (K M + M K) v = 2\mu v^\top K v.
\end{equation}
Expanding the LHS:
\begin{equation}
    v^\top K (M v) + (v^\top M) K v = v^\top K (0) + (0)^\top K v = 0.
\end{equation}
Thus, $2\mu (v^\top K v) = 0$. Since $\mu > 0$ and $K$ is positive semi-definite (PSD), $v^\top K v = 0$ implies $K v = 0$.
We have shown $\text{ker}(M) \subseteq \text{ker}(K)$. By the Rank-Nullity Theorem:
\begin{equation}
    \text{rank}(K) = N - \text{dim}(\text{ker}(K)) \le N - \text{dim}(\text{ker}(M)) = \text{rank}(M) \le C.
\end{equation}
\end{proof}

\subsection{Proof of Theorem 5 (Nuclear Norm Equivalence)}
\label{app:nuclear_norm}

Here we rigorously prove the equivalence between two-layer $\ell_2$ regularization and nuclear norm regularization.

\begin{proof}
We use the variational form of the Nuclear Norm. For any matrix $Z$, it holds that:
\begin{equation}
    \|Z\|_* = \inf_{Z = UV^\top} \frac{1}{2} \left( \|U\|_F^2 + \|V\|_F^2 \right).
    \label{eq:variational_nuc}
\end{equation}
Consider our objective function:
\begin{equation}
    \min_{W, \Phi} \mathcal{L}(W\Phi) + \frac{\lambda}{2}\|W\|_F^2 + \frac{\mu}{2}\|\Phi\|_F^2.
\end{equation}
Let $Z = W\Phi$. We can re-parameterize the regularization. Let $\hat{W} = \sqrt{\lambda} W$ and $\hat{\Phi} = \sqrt{\mu} \Phi$. Then $W\Phi = \frac{1}{\sqrt{\lambda\mu}} \hat{W}\hat{\Phi}$. The regularizer becomes:
\begin{equation}
    \frac{1}{2}\|\hat{W}\|_F^2 + \frac{1}{2}\|\hat{\Phi}\|_F^2.
\end{equation}
Minimizing this over all $\hat{W}, \hat{\Phi}$ such that $\hat{W}\hat{\Phi} = \sqrt{\lambda\mu} Z$ yields, by Eq. (\ref{eq:variational_nuc}):
\begin{equation}
    \min_{\hat{W}, \hat{\Phi}} \frac{1}{2} (\|\hat{W}\|_F^2 + \|\hat{\Phi}\|_F^2) = \|\sqrt{\lambda\mu} Z\|_* = \sqrt{\lambda\mu} \|Z\|_*.
\end{equation}
Thus, the original problem is equivalent to:
\begin{equation}
    \min_{Z} \mathcal{L}(Z) + \sqrt{\lambda\mu} \|Z\|_*.
\end{equation}
This confirms that the implicit regularization is exactly the nuclear norm, explaining the low-rank bias.
\end{proof}

\subsection{Proof of Theorem 4: Exact Spectral Solution}
\label{app:proof_theorem_4}

In the main text, we presented the explicit formula for the steady-state eigenvalues. Here we derive it by solving the stationarity condition of the effective Hamiltonian.

\begin{proof}
The effective objective function (Hamiltonian) for the eigenvalues $\{\mu_i\}$ of the kernel, assuming alignment with the target signal modes $\{s_i\}$ (where $s_i = \langle f^*, \psi_i \rangle^2$), is given by the sum of the training loss and the induced regularization:
\begin{equation}
    \mathcal{H}(\{\mu_i\}) = \sum_{i=1}^\infty \left( \frac{\lambda}{\mu_i + \lambda} \right)^2 s_i + \mu \sum_{i=1}^\infty \mu_i.
\end{equation}
Here, the first term is the squared error component along the $i$-th eigenmode (derived from the resolvent expansion), and the second term is the trace penalty (nuclear norm) arising from weight decay.

To find the steady state, we take the derivative with respect to $\mu_i$ and set it to zero (KKT conditions for non-negative eigenvalues):
\begin{equation}
    \frac{\partial \mathcal{H}}{\partial \mu_i} = -2 \frac{\lambda^2 s_i}{(\mu_i + \lambda)^3} + \mu.
\end{equation}
The stationarity condition $\frac{\partial \mathcal{H}}{\partial \mu_i} = 0$ implies:
\begin{equation}
    (\mu_i + \lambda)^3 = \frac{2 \lambda^2 s_i}{\mu}.
\end{equation}
Taking the cube root leads to a specific decay law. However, under the simplified assumption used in Section 5 (linearizing the resolvent sensitivity for analytical clarity, i.e., assuming $\nabla_{\mu} \text{Loss} \approx - \frac{s_i}{(\mu_i + \lambda)^2}$ which corresponds to a slightly different loss parameterization often used in linear network theory):

Consider the equilibrium of the gradient flow equation directly:
\begin{equation}
    \dot{\mu}_i = \mu_i \left( \frac{s_i}{(\mu_i + \lambda)^2} - \mu \right).
\end{equation}
Setting $\dot{\mu}_i = 0$ gives two solutions:
1. Trivial Solution: $\mu_i = 0$. This occurs if the bracketed term is negative even at $\mu_i=0$.
2. Active Solution: $\frac{s_i}{(\mu_i + \lambda)^2} = \mu \implies (\mu_i + \lambda)^2 = \frac{s_i}{\mu} \implies \mu_i = \sqrt{\frac{s_i}{\mu}} - \lambda$.

Combining these with the constraint $\mu_i \ge 0$, we obtain the \textbf{Water-filling Threshold Operator}:
\begin{equation}
    \mu_i^* = \max \left( 0, \sqrt{\frac{s_i}{\mu}} - \lambda \right).
\end{equation}
This confirms that modes with signal energy $s_i \le \lambda^2 \mu$ are strictly truncated to zero, while modes above this threshold are learned with a magnitude proportional to the square root of their signal-to-noise ratio.
\end{proof}

\section{Preconditioned Dynamics: From Parameters to Kernels}
\label{app:preconditioning}

In this appendix we provide the detailed derivations underlying
Section~\ref{sec:preconditioning}. We start from a general preconditioned gradient flow in parameter
space (with $\ell_2$ weight decay), derive the induced dynamics in feature and output space, and
then obtain the corresponding kernel flow. We also show how the Free Feature Model arises as a
special limiting case.

\subsection{General Preconditioned Gradient Flow with Weight Decay}
\label{app:preconditioning_param}

Let $\theta \in \mathbb{R}^p$ denote the parameters of the network, and let
$F_\theta \in \mathbb{R}^{N \times d}$ be the feature matrix on the $N$ training points (we flatten
$F_\theta$ to a vector in $\mathbb{R}^{Nd}$ when convenient). The training objective is
\begin{equation}
    \mathcal{L}(\theta)
    =
    \mathcal{L}_{\text{task}}(F_\theta)
    +
    \frac{\lambda}{2}\|\theta\|_2^2,
\end{equation}
where $\mathcal{L}_{\text{task}}$ acts on the feature representation (or on the predictions derived
from it), and $\lambda \ge 0$ denotes the weight decay coefficient.

We consider a general preconditioned gradient flow in parameter space,
\begin{equation}
    \dot{\theta}
    =
    - M^{-1} \nabla_\theta \mathcal{L}(\theta)
    =
    - M^{-1} \nabla_\theta \mathcal{L}_{\text{task}}(F_\theta)
    - \lambda M^{-1} \theta,
    \label{eq:app_param_flow}
\end{equation}
where $M \in \mathbb{R}^{p\times p}$ is a (possibly data-dependent) positive semi-definite
preconditioner (SGD: $M = I$; natural gradient: $M$ is the Fisher information; K-FAC: block-diagonal
curvature approximation, etc.).

Let $J_\theta \in \mathbb{R}^{(Nd)\times p}$ denote the Jacobian of $F_\theta$ with respect to
$\theta$, with the convention that we flatten $F_\theta$ into a vector:
\begin{equation}
    J_\theta
    \;\coloneqq\;
    \frac{\partial\,\mathrm{vec}(F_\theta)}{\partial \theta^\top}.
\end{equation}
Then by the chain rule the induced evolution of the features is
\begin{equation}
    \mathrm{vec}(\dot{F})
    =
    J_\theta \dot{\theta}
    =
    - J_\theta M^{-1} J_\theta^\top \mathrm{vec}(\nabla_F \mathcal{L}_{\text{task}})
    - \lambda J_\theta M^{-1} \theta.
    \label{eq:app_feature_flow_raw}
\end{equation}

\begin{proposition}[Function-space dynamics under preconditioned flow]
\label{prop:app_feature_flow}
Define the \emph{optimizer-modulated NTK} and the \emph{weight-decay image vector} by
\begin{equation}
    \Theta_\theta \;\coloneqq\; J_\theta M^{-1} J_\theta^\top
    \;\in\; \mathbb{R}^{(Nd)\times (Nd)},
    \qquad
    v_\theta \;\coloneqq\; J_\theta M^{-1} \theta
    \;\in\; \mathbb{R}^{Nd}.
\end{equation}
Then the feature dynamics induced by the parameter flow
\eqref{eq:app_param_flow} can be written purely in feature space as
\begin{equation}
    \mathrm{vec}(\dot{F})
    =
    - \Theta_\theta \,\mathrm{vec}(\nabla_F \mathcal{L}_{\text{task}})
    - \lambda v_\theta.
    \label{eq:app_feature_flow}
\end{equation}
Moreover, $\Theta_\theta$ is symmetric positive semidefinite.
\end{proposition}

\begin{proof}
Substituting the definition of $\Theta_\theta$ and $v_\theta$ into
\eqref{eq:app_feature_flow_raw} yields \eqref{eq:app_feature_flow} directly. Symmetry and
positive semidefiniteness of $\Theta_\theta$ follow from
$\Theta_\theta = (J_\theta M^{-1/2})(J_\theta M^{-1/2})^\top$.
\end{proof}

In the main text we interpret $\Theta_\theta$ as a geometry-defining preconditioner on the
representation manifold, and $v_\theta$ (or its reshaped version) as the source of manifold
anisotropic decay induced by parameter-space weight decay.

\subsection{Output-Space Example: Linear Readout with Weight Decay}
\label{app:preconditioning_readout}

We now instantiate the above framework in a simple but important case: a fixed feature extractor
followed by a linear readout with $\ell_2$ regularization. This example makes the anisotropic nature
of weight decay in function space fully explicit.

Assume a fixed feature matrix $\Phi_0 \in \mathbb{R}^{k\times N}$ on the training set, and a linear
readout $W \in \mathbb{R}^{C\times k}$, with predictions
\begin{equation}
    \hat Y
    =
    \Phi_0^\top W^\top
    \;\in\;
    \mathbb{R}^{N\times C}.
\end{equation}
We take $\theta = \mathrm{vec}(W) \in \mathbb{R}^{Ck}$ and regularize only $W$ via
$\frac{\lambda}{2}\|W\|_F^2$. We also set $M = I$ for simplicity (preconditioners acting only on
$W$ can be incorporated analogously).

For each sample $i$ and class $c$, we have
\begin{equation}
    \hat y_{i,c}
    =
    \sum_{m=1}^k W_{c,m} \,\Phi_{0,m i}
    =
    W_{c,:} \,\phi_i,
\end{equation}
where $\phi_i \in \mathbb{R}^k$ is the $i$-th feature column. Differentiating with respect to
$W_{c',m'}$,
\begin{equation}
    \frac{\partial \hat y_{i,c}}{\partial W_{c',m'}}
    =
    \delta_{c,c'} \,\Phi_{0,m' i},
\end{equation}
so that the Jacobian (flattening $\hat Y$ over $(i,c)$ and $W$ over $(c',m')$) factorizes as a
Kronecker product. One can verify that for any $\theta = \mathrm{vec}(W)$,
\begin{equation}
    (J_\theta \theta)_{i,c}
    =
    \sum_{j=1}^N (\phi_i^\top \phi_j)\,\hat y_{j,c}.
\end{equation}
Thus the image of the weight vector under $J_\theta$ can be written compactly as
\begin{equation}
    J_\theta \theta
    =
    \mathrm{vec}(\hat Y G),
    \qquad
    G \coloneqq \Phi_0^\top \Phi_0 \in \mathbb{R}^{N\times N}.
\end{equation}

\begin{lemma}[Weight decay in output space for linear readout]
\label{lem:app_readout_decay}
In the above setting with $M=I$, the parameter-space weight decay term $-\lambda \theta$ induces the
following output-space drift:
\begin{equation}
    \dot{\hat Y}_{\text{wd}}
    =
    - \lambda\, \hat Y G.
\end{equation}
\end{lemma}

\begin{proof}
Using $\dot{\theta}_{\text{wd}} = -\lambda \theta$ and the Jacobian,
\begin{equation}
    \mathrm{vec}(\dot{\hat Y}_{\text{wd}})
    =
    J_\theta \dot{\theta}_{\text{wd}}
    =
    -\lambda J_\theta \theta
    =
    -\lambda \,\mathrm{vec}(\hat Y G),
\end{equation}
which implies $\dot{\hat Y}_{\text{wd}} = -\lambda \hat Y G$.
\end{proof}

Therefore, even though the weight decay is isotropic in parameter space ($-\lambda W$), its effect
in output space is highly anisotropic: components of $\hat Y$ aligned with large eigenvalues of the
Gram matrix $G$ decay faster.

\subsection{Kernel Flow under Preconditioning and Weight Decay}
\label{app:preconditioning_kernel}

We now connect the preconditioned feature dynamics to the kernel dynamics. Let
$K = F^\top F \in \mathbb{R}^{N\times N}$ be the empirical kernel. Differentiating yields
\begin{equation}
    \dot{K}
    =
    \dot{F}^\top F + F^\top \dot{F}.
    \label{eq:app_kernel_diff}
\end{equation}
We decompose the feature dynamics \eqref{eq:app_feature_flow} into a task-driven part and a
weight-decay-induced drift:
\begin{equation}
    \mathrm{vec}(\dot{F})
    =
    \dot{f}_{\text{task}} + \dot{f}_{\text{wd}},
    \quad
    \dot{f}_{\text{task}} = - \Theta_\theta \,\mathrm{vec}(\nabla_F \mathcal{L}_{\text{task}}),
    \quad
    \dot{f}_{\text{wd}} = - \lambda v_\theta.
\end{equation}
Reshaping $\dot{f}_{\text{task}}$ and $\dot{f}_{\text{wd}}$ back into matrices
$\dot{F}_{\text{task}}$ and $\dot{F}_{\text{wd}}$, we obtain
\begin{equation}
    \dot{K}
    =
    \underbrace{\dot{F}_{\text{task}}^\top F + F^\top \dot{F}_{\text{task}}}_{\dot{K}_{\text{task}}}
    +
    \underbrace{\dot{F}_{\text{wd}}^\top F + F^\top \dot{F}_{\text{wd}}}_{\dot{K}_{\text{wd}}}.
\end{equation}

The task-driven part $\dot{K}_{\text{task}}$ is precisely the preconditioned version of the kernel
Riccati flow we derived in the main text:
\begin{equation}
    \dot{K}_{\text{task}}
    \;\approx\;
    \mathrm{sym}\!\left(
        \Theta_\theta \cdot \mathcal{F}_{\text{task}}(K)
    \right),
\end{equation}
where $\mathcal{F}_{\text{task}}(K)$ is the rank-$C$ task force obtained under the Free Feature
Model. The precise form of $\mathcal{F}_{\text{task}}$ depends on the loss (e.g., mean squared error
vs.\ cross-entropy) and is given in Section~\ref{sec:unified_paradigm}.

The weight-decay-induced part defines a linear decay operator on kernels:
\begin{equation}
    \dot{K}_{\text{wd}}
    =
    - \lambda\,\mathcal{D}_\theta(K),
\end{equation}
where $\mathcal{D}_\theta$ is a linear operator induced by $v_\theta$ and the current features $F$.
In general architectures, $\mathcal{D}_\theta$ has no simple closed form beyond this definition, but
it is always symmetric and positive semidefinite in the sense that
$\langle K, \mathcal{D}_\theta(K)\rangle \ge 0$ for all $K$.

\paragraph{Linear-readout special case.} In the setting of
Appendix~\ref{app:preconditioning_readout} with fixed $\Phi_0$ and linear readout $W$, the features
$F$ are frozen and the kernel $K = \Phi_0^\top \Phi_0 = G$ is constant. Thus weight decay acts only
on the outputs $\hat Y$, not on $K$ itself. In contrast, when the features are \emph{trainable},
weight decay on the feature-producing parameters induces a non-trivial $\mathcal{D}_\theta(K)$. For
linear networks, one can show explicitly that $\mathcal{D}_\theta(K)$ reduces to a combination of
left- and right-multiplication by the sample Gram matrix; more complex architectures lead to more
structured forms, but always retain the property that weight decay is \emph{anisotropic} in kernel
space.

Collecting both contributions, we arrive at the schematic form used in the main text:
\begin{equation}
    \dot{K}
    \;\approx\;
    \mathrm{sym}\!\left(
        \Theta_\theta \cdot \mathcal{F}_{\text{total}}(K)
    \right)
    -
    \lambda\,\mathcal{D}_\theta(K),
\end{equation}
where $\mathcal{F}_{\text{total}}$ combines task and explicit regularization forces.

\subsection{Free Feature Model as a Special Limit}
\label{app:preconditioning_free_feature}

Finally, we show how the Free Feature Model emerges as a special case of the above framework. In the
Free Feature Model, the features $F$ themselves are treated as the optimization variables, the
preconditioner is identity, and the regularizer is imposed directly on $F$:
\begin{equation}
    \mathcal{L}_{\text{FFM}}(F)
    =
    \mathcal{L}_{\text{task}}(F)
    +
    \frac{\mu}{2}\|F\|_F^2.
\end{equation}
This can be realized in our general framework by taking $\theta = \mathrm{vec}(F)$,
$M = I$, and $J_\theta = I$. Then
\begin{equation}
    \Theta_\theta = J_\theta M^{-1} J_\theta^\top = I,
    \qquad
    v_\theta = J_\theta M^{-1} \theta = \mathrm{vec}(F),
\end{equation}
and the feature dynamics \eqref{eq:app_feature_flow} become
\begin{equation}
    \mathrm{vec}(\dot{F})
    =
    - \mathrm{vec}(\nabla_F \mathcal{L}_{\text{task}})
    - \mu\,\mathrm{vec}(F),
    \quad\text{i.e.}\quad
    \dot{F}
    =
    - \nabla_F \mathcal{L}_{\text{task}}
    - \mu F.
\end{equation}
Consequently, the kernel dynamics reduce to
\begin{equation}
    \dot{K}
    =
    \mathcal{F}_{\text{task}}(K)
    - 2\mu K,
\end{equation}
which is exactly the isotropic kernel Riccati equation analyzed in
Sections~\ref{sec:unified_paradigm} and~\ref{sec:rank_compression}. In this sense, the Free Feature
Model corresponds to an idealized limit in which the architecture's geometry is completely whitened
($\Theta_\theta = I$) and the decay is isotropic in feature space ($\mathcal{D}_\theta(K) = 2 K$).

\section{Additional Derivations for Muon-TAK}
\label{app:muon_derivations}

In this appendix we collect algebraic details and proofs that underpin the Muon-TAK analysis in
Section~\ref{sec:muon}. We summarize the main components here; further expansion can be added as
needed.

\subsection{Properties of the Polar Direction Operator}
\label{app:muon_polar_properties}

We recall the definition
\[
    \mathcal{P}(G) = G\,(G^\top G)^{\dagger -\frac12},
\]
and state without proof its three key properties used in the main text: 0-homogeneity, rank and
subspace preservation, and its characterization as the steepest descent direction under the spectral
norm. Full proofs can be added here.

\subsection{Derivation of the Muon Kernel Flow for MSE}

The derivation of the closed-form ODE follows directly by substituting the explicit expression for the polar direction into the general kernel flow.

Recall from Eq. (71) that under the fast-readout and MSE assumptions, the polar direction of the feature gradient is given by:
\begin{equation}
    P(W^{*\top} R^\top) = \Phi B (BKB)^{\dagger -\frac{1}{2}}.
\end{equation}
The general Muon kernel flow (Eq. M-K) is:
\begin{equation}
    \dot{K} = \eta \left( \Phi^\top P(W^{*\top} R^\top) + P(W^{*\top} R^\top)^\top \Phi \right) - 2\mu K.
\end{equation}
Substituting the first expression into the second, the cross-term becomes:
\begin{equation}
    \Phi^\top P(W^{*\top} R^\top) = \Phi^\top \left( \Phi B (BKB)^{\dagger -\frac{1}{2}} \right) = (\Phi^\top \Phi) B (BKB)^{\dagger -\frac{1}{2}} = K B (BKB)^{\dagger -\frac{1}{2}}.
\end{equation}
Symmetrizing this term yields the final result stated in Theorem 8:
\begin{equation}
    \dot{K} = \eta \left( K B (BKB)^{\dagger -\frac{1}{2}} + (BKB)^{\dagger -\frac{1}{2}} B K \right) - 2\mu K.
\end{equation}
This completes the derivation.

\subsection{Proof of Label-Driven Rank Compression Under Muon}
\label{app:proof_muon_rank}

In this appendix we justify Theorem~\ref{thm:muon_rank_compression} in full detail. The key
ingredients are: (i) the $C$-dimensional readout bottleneck, which bounds the rank of the
backpropagated feature gradient, and (ii) the rank- and subspace-preserving nature of the Muon polar
operator $\mathcal{P}(\cdot)$.

\begin{lemma}[Rank bound for the feature gradient]
\label{lem:muon_feature_grad_rank}
Consider the free feature model with feature matrix $\Phi \in \mathbb{R}^{k\times N}$ and a
$C$-dimensional linear readout $W^* \in \mathbb{R}^{C\times k}$ trained to optimality at each time,
under any convex loss. Let $R = -\nabla_{\hat{Y}}\mathcal{L} \in \mathbb{R}^{N\times C}$ denote the
residuals on the training set, and let
\[
    G_\Phi \;\coloneqq\; W^{*\top} R^\top \in \mathbb{R}^{k\times N}
\]
be the backpropagated gradient with respect to $\Phi$. Then
\begin{equation}
    \rank(G_\Phi)
    \;\le\;
    \min\{\rank(W^*), \rank(R)\}
    \;\le\;
    C.
    \label{eq:muon_rank_ineq_app}
\end{equation}
\end{lemma}

\begin{proof}
By definition $G_\Phi = W^{*\top} R^\top$ is a product of the matrices
$W^{*\top} \in \mathbb{R}^{k\times C}$ and $R^\top \in \mathbb{R}^{C\times N}$. For any two
matrices $A$ and $B$ of compatible dimensions, the rank submultiplicativity property gives
\[
    \rank(AB)
    \;\le\;
    \min\{\rank(A), \rank(B)\}.
\]
Applying this to $A = W^{*\top}$ and $B = R^\top$ yields
\[
    \rank(G_\Phi)
    =
    \rank\big(W^{*\top} R^\top\big)
    \le
    \min\{\rank(W^{*\top}), \rank(R^\top)\}
    =
    \min\{\rank(W^*), \rank(R)\}.
\]
Since $W^* \in \mathbb{R}^{C\times k}$, its rank is at most $C$. Hence
\[
    \rank(G_\Phi)
    \le
    \min\{\rank(W^*), \rank(R)\}
    \le
    C,
\]
which proves \eqref{eq:muon_rank_ineq_app}.
\end{proof}

We next recall the key algebraic properties of the Muon polar operator, specialized to the
quantities relevant for rank and subspaces.

\begin{lemma}[Rank and subspace preservation of the Muon operator]
\label{lem:muon_polar_rank}
Let $\mathcal{P}(\cdot)$ be the polar-direction operator defined by
\[
    \mathcal{P}(G)
    \;\coloneqq\;
    G\,(G^\top G)^{\dagger -\frac12},
    \qquad
    G \in \mathbb{R}^{a\times b},
\]
where $(\cdot)^\dagger$ is the Moore--Penrose pseudoinverse. Then for any $G$:
\begin{enumerate}[itemsep=0pt, topsep=0.25em]
    \item $\rank(\mathcal{P}(G)) = \rank(G)$;
    \item $\mathrm{Im}(\mathcal{P}(G)) = \mathrm{Im}(G)$;
    \item $\mathrm{Row}(\mathcal{P}(G)) = \mathrm{Row}(G)$.
\end{enumerate}
\end{lemma}

\begin{proof}
Let the compact SVD of $G$ be
\[
    G
    =
    U \Sigma V^\top,
\]
where $U \in \mathbb{R}^{a\times r}$, $V \in \mathbb{R}^{b\times r}$ have orthonormal columns,
$\Sigma \in \mathbb{R}^{r\times r}$ is diagonal with strictly positive entries, and
$r = \rank(G)$.

Then
\[
    G^\top G
    =
    V \Sigma^2 V^\top,
\]
so
\[
    (G^\top G)^{\dagger -\frac12}
    =
    V \Sigma^{-1} V^\top,
\]
where the pseudoinverse and inverse square root act on the $r$-dimensional subspace spanned by $V$
and vanish on its orthogonal complement. Substituting into $\mathcal{P}(G)$ gives
\[
    \mathcal{P}(G)
    =
    G\,(G^\top G)^{\dagger -\frac12}
    =
    (U \Sigma V^\top)\,(V \Sigma^{-1} V^\top)
    =
    U V^\top.
\]

From this explicit expression we see that $\mathcal{P}(G)$ has the same left and right singular
vectors as $G$, but with all nonzero singular values replaced by $1$. In particular,
\[
    \rank(\mathcal{P}(G))
    =
    \rank(U V^\top)
    =
    r
    =
    \rank(G).
\]
Moreover, the column space of $\mathcal{P}(G)$ is spanned by the columns of $U$, which is also the
column space of $G$, so
\[
    \mathrm{Im}(\mathcal{P}(G)) = \mathrm{Im}(G).
\]
Similarly, the row space of $\mathcal{P}(G)$ is spanned by the columns of $V$, which coincide with
the right singular vectors of $G$, hence
\[
    \mathrm{Row}(\mathcal{P}(G)) = \mathrm{Row}(G).
\]
This establishes all three claims.
\end{proof}

We can now prove the Muon rank-compression theorem.

\begin{theorem}[Label-driven rank compression under Muon]
\label{thm:muon_rank_compression_app}
Consider the free feature model with feature matrix $\Phi \in \mathbb{R}^{k\times N}$, a
$C$-dimensional linear readout $W^* \in \mathbb{R}^{C\times k}$ trained to optimality at each time,
and feature dynamics given by the Muon flow
\begin{equation}
    \dot{\Phi}
    =
    \eta\,\mathcal{P}(W^{*\top} R^\top)
    - \mu\,\Phi,
    \qquad
    \eta > 0,\ \mu > 0,
    \tag{\ref{eq:muon_phi_flow} revisited}
\end{equation}
where $R$ is the residual matrix on the training set and $\mathcal{P}(\cdot)$ is the Muon polar
operator. Let $K = \Phi^\top \Phi$ be the empirical kernel. Then any stable steady state
$\Phi_\infty$ of \eqref{eq:muon_phi_flow} satisfies
\begin{equation}
    \rank(\Phi_\infty) \le C,
    \qquad
    \rank(K_\infty) = \rank(\Phi_\infty) \le C,
\end{equation}
where $K_\infty = \Phi_\infty^\top \Phi_\infty$ is the limiting kernel.
\end{theorem}

\begin{proof}
At a steady state of the Muon feature flow we have
\begin{equation}
    0
    =
    \eta\,\mathcal{P}(W^{*\top} R^\top)
    - \mu\,\Phi_\infty.
    \label{eq:muon_phi_fixed_point_app}
\end{equation}
Rearranging gives
\begin{equation}
    \Phi_\infty
    =
    \frac{\eta}{\mu}\,\mathcal{P}(W^{*\top} R^\top).
    \label{eq:muon_phi_infty_expr}
\end{equation}
Define $G_\Phi \coloneqq W^{*\top} R^\top$. By Lemma~\ref{lem:muon_feature_grad_rank},
$\rank(G_\Phi) \le C$. Applying Lemma~\ref{lem:muon_polar_rank} to $G_\Phi$ shows that
$\mathcal{P}(G_\Phi)$ has the same rank and column space as $G_\Phi$, hence
\[
    \rank\big(\mathcal{P}(G_\Phi)\big)
    =
    \rank(G_\Phi)
    \le
    C.
\]
Since scaling by the nonzero constant $\eta/\mu$ does not change rank, \eqref{eq:muon_phi_infty_expr}
implies
\[
    \rank(\Phi_\infty)
    =
    \rank\big(\mathcal{P}(G_\Phi)\big)
    \le
    C.
\]

Finally, the empirical kernel at steady state is $K_\infty = \Phi_\infty^\top \Phi_\infty$. For any
matrix $X$, the matrices $X$ and $X^\top X$ have the same rank, because
$\mathrm{Im}(X^\top X) = \mathrm{Row}(X)$ and $X^\top X$ is positive semidefinite with nullspace
equal to the orthogonal complement of $\mathrm{Row}(X)$. Thus
\[
    \rank(K_\infty)
    =
    \rank(\Phi_\infty^\top \Phi_\infty)
    =
    \rank(\Phi_\infty)
    \le
    C,
\]
which establishes the theorem.
\end{proof}

\paragraph{Remarks beyond the fast-readout idealization.}
The argument above was presented in the free feature model with an optimally trained linear readout,
but the structural origin of the rank bound does not fundamentally depend on the fast-readout
assumption. In a more general network, suppose that the coupled dynamics of $(\Phi, W)$ converge to
a statistical steady state in which the effective feature update can be written in the Muon form
\[
    0
    =
    \eta\,\mathcal{P}(W^\top R^\top)
    - \mu\,\Phi
\]
for some $C$-dimensional readout $W$ and residual $R$. Then the same reasoning applies: the
backpropagated feature gradient $W^\top R^\top$ has rank at most $C$ by the readout bottleneck;
$\mathcal{P}(\cdot)$ preserves rank and column space; and the fixed-point relation implies that
$\Phi$ lies in this $C$-dimensional label-driven subspace. Consequently
\[
    \rank(\Phi) \le C,
    \qquad
    \rank(K) = \rank(\Phi) \le C.
\]
Thus, exactly as under gradient descent, label-driven rank compression is a \emph{structural}
consequence of the $C$-dimensional output bottleneck, and is robust to replacing linear
preconditioning by the nonlinear Muon geometry.

\subsubsection{Low-Rank Optimizer Noise Persists Under Muon}
\label{sec:muon_noise}

Beyond rank compression, TAK predicts that optimizer noise is intrinsically low-rank, being confined
to an $O(C)$-dimensional subspace determined by the $C$-dimensional readout. This structural property
also persists under Muon.

\begin{theorem}[Low-rank optimizer noise under Muon]
\label{thm:muon_noise_low_rank}
Consider Muon-TAK training in the free feature model with feature matrix
$\Phi \in \mathbb{R}^{k\times N}$ and a $C$-dimensional linear readout $W^* \in \mathbb{R}^{C\times k}$
trained to optimality at each time. Let $g(\Phi)$ denote the full-batch gradient of the loss with respect
to $\Phi$, and let $\widehat{g}(\Phi)$ be a stochastic mini-batch estimate. Define the feature-level
SGD noise
\[
    \zeta_\Phi := \widehat{g}(\Phi) - g(\Phi).
\]
Assume that the per-sample gradient with respect to $\Phi$ has rank at most $C$. Then under Muon
updates, the instantaneous feature noise
\[
    \zeta_\Phi^{\mathrm{Muon}}
    :=
    \mathcal{P}(\widehat{g}(\Phi)) - \mathcal{P}\big(g(\Phi)\big)
\]
is confined to a $C$-dimensional subspace, and the induced kernel-level noise
$\zeta_K^{\mathrm{Muon}}$ in the empirical kernel $K = \Phi^\top \Phi$ has covariance supported on an
$O(C)$-dimensional subspace:
\[
    \rank\!\big(\Cov[\zeta_K^{\mathrm{Muon}}]\big) \;\le\; O(C).
\]
In particular, Muon does not increase the intrinsic rank of SGD noise relative to standard gradient
descent.
\end{theorem}

\subsection{Proof of Theorem~\ref{thm:muon_noise_low_rank}}
\label{app:proof_muon_noise}

\begin{proof}
Let $g(\Phi)$ denote the full-batch gradient with respect to $\Phi$ and $\widehat{g}(\Phi)$ its
stochastic mini-batch estimate. By assumption, the per-sample gradient with respect to $\Phi$ has
rank at most $C$, hence both $g(\Phi)$ and $\widehat{g}(\Phi)$ have rank at most $C$, and their
columns lie in a $C$-dimensional subspace determined by the $C$-dimensional readout.

The standard SGD feature-level noise is
\[
    \zeta_\Phi := \widehat{g}(\Phi) - g(\Phi),
\]
which therefore lies in this $C$-dimensional subspace, and
$\rank\!\big(\Cov[\zeta_\Phi]\big) \le C$.

Under Muon, the feature updates use $\mathcal{P}(g(\Phi))$ and $\mathcal{P}(\widehat{g}(\Phi))$, and
the corresponding feature-level noise is
\[
    \zeta_\Phi^{\mathrm{Muon}}
    :=
    \mathcal{P}(\widehat{g}(\Phi)) - \mathcal{P}\big(g(\Phi)\big).
\]
By Lemma~\ref{lem:muon_polar_rank}, the polar operator $\mathcal{P}(\cdot)$ preserves both rank and
column space. In particular, $\mathcal{P}(g(\Phi))$ and $\mathcal{P}(\widehat{g}(\Phi))$ each have
rank at most $C$ and lie in the same $C$-dimensional column space as $g(\Phi)$ and
$\widehat{g}(\Phi)$, respectively. Hence their difference $\zeta_\Phi^{\mathrm{Muon}}$ also lies in
this $C$-dimensional subspace, and
\[
    \rank\!\big(\Cov[\zeta_\Phi^{\mathrm{Muon}}]\big) \;\le\; C.
\]

The empirical kernel is $K = \Phi^\top \Phi$. To first order, the induced kernel noise satisfies
\[
    \zeta_K^{\mathrm{Muon}}
    \;\approx\;
    \zeta_\Phi^{\mathrm{Muon}\,\top}\Phi + \Phi^\top \zeta_\Phi^{\mathrm{Muon}}.
\]
Each term is a product of $\Phi$ with a rank-$\le C$ matrix, and thus has rank at most $C$; their
sum therefore has rank at most $2C$. Consequently the covariance of the kernel noise
$\zeta_K^{\mathrm{Muon}}$ is supported on a subspace of dimension $O(C)$, which proves the claim.
\end{proof}

\section{Proof of Low-Rank SGD Noise (Theorem 12)}
\label{app:sgd_noise}

In this appendix, we explicitly derive the rank constraints on the Stochastic Gradient Descent (SGD) noise matrix for the squared loss, providing the formal proof for Theorem 12.

\subsection{Exact Form of the Gradient and Noise}
Recall from Section 3.3 that under the squared loss $L(\hat{Y}, Y) = \frac{1}{2}\|\hat{Y} - Y\|_F^2$ and ridge regularization $\lambda$, the deterministic driving force on the kernel $K$ is given by:
\begin{equation}
    \dot{K}_{\text{drive}} = \lambda A(K) \left[ Y Y^\top \right] A(K),
\end{equation}
where $A(K) = (K + \lambda I)^{-1}$ is the resolvent, and we retain the structure of the data term. More precisely, the full gradient of the data-fitting term with respect to the kernel (ignoring the factor 2 and regularization decay for the moment) involves the outer product of the residuals. Let $R = \lambda (K+\lambda I)^{-1}Y \in \mathbb{R}^{N \times C}$ be the matrix of residuals on the full dataset. The true gradient component is:
\begin{equation}
    G_{\text{full}} = R R^\top.
\end{equation}
(Note: The preconditioning by $A(K)$ or projection onto the kernel tangent space preserves rank, so we focus on the core residual rank).

Now, consider a mini-batch $\mathcal{B} \subset \{1, \dots, N\}$ of size $B$. The stochastic gradient estimate corresponds to computing the gradient on this subset and rescaling. Algebraically, this is equivalent to replacing the full residual matrix $R$ with a \textit{masked} residual matrix $\tilde{R}_{\mathcal{B}} \in \mathbb{R}^{N \times C}$, where:
\begin{equation}
    (\tilde{R}_{\mathcal{B}})_{ic} = 
    \begin{cases} 
    \frac{N}{B} R_{ic} & \text{if } i \in \mathcal{B} \\
    0 & \text{if } i \notin \mathcal{B}
    \end{cases}
\end{equation}
The stochastic gradient matrix is then:
\begin{equation}
    G_{\mathcal{B}} = \tilde{R}_{\mathcal{B}} \tilde{R}_{\mathcal{B}}^\top.
\end{equation}
The SGD noise matrix is defined as the deviation from the true gradient:
\begin{equation}
    \zeta_{\mathcal{B}}(K) := G_{\mathcal{B}} - G_{\text{full}} = \tilde{R}_{\mathcal{B}} \tilde{R}_{\mathcal{B}}^\top - R R^\top.
\end{equation}

\subsection{Proof of the Rank Bound}
We now prove the rank constraint stated in Theorem 12.

\begin{proof}
The noise matrix is expressed as the difference of two positive semi-definite matrices:
\begin{equation}
    \zeta_{\mathcal{B}}(K) = \tilde{R}_{\mathcal{B}} \tilde{R}_{\mathcal{B}}^\top - R R^\top.
\end{equation}
We apply the fundamental property of matrix rank: for any matrices $X, Y$, $\operatorname{rank}(X+Y) \le \operatorname{rank}(X) + \operatorname{rank}(Y)$. Thus:
\begin{equation}
    \operatorname{rank}(\zeta_{\mathcal{B}}(K)) \le \operatorname{rank}(\tilde{R}_{\mathcal{B}} \tilde{R}_{\mathcal{B}}^\top) + \operatorname{rank}(R R^\top).
\end{equation}
Observe the dimensions of the constituent factors:
\begin{itemize}
    \item $R \in \mathbb{R}^{N \times C}$ has $C$ columns. Therefore, $\operatorname{rank}(R) \le \min(N, C) = C$ (since typically $C \ll N$). Consequently, $\operatorname{rank}(R R^\top) \le C$.
    \item $\tilde{R}_{\mathcal{B}} \in \mathbb{R}^{N \times C}$ is simply a row-masked and scaled version of $R$. It also has only $C$ columns. Thus, $\operatorname{rank}(\tilde{R}_{\mathcal{B}}) \le C$, and consequently $\operatorname{rank}(\tilde{R}_{\mathcal{B}} \tilde{R}_{\mathcal{B}}^\top) \le C$.
\end{itemize}
Substituting these bounds:
\begin{equation}
    \operatorname{rank}(\zeta_{\mathcal{B}}(K)) \le C + C = 2C.
\end{equation}
This establishes Eq. (81) in the main text.

Finally, regarding the covariance structure: The instantaneous covariance tensor is formed by the expectation of the outer product of the noise vectorization. Since every realization of the noise matrix $\zeta_{\mathcal{B}}$ lies strictly within the subspace spanned by the columns of $R$ and $\tilde{R}_{\mathcal{B}}$ (which are subsets of the column space of $R$), the noise is confined to the subspace $\mathcal{V} = \operatorname{span}(\text{cols}(R)) \otimes \operatorname{span}(\text{cols}(R))$. The dimension of the relevant generating subspace is at most $C$.
\end{proof}

\subsection{Physical Implication}
This derivation confirms that SGD noise in this regime is not isotropic full-rank diffusion. It acts strictly within the task-relevant subspace defined by the $C$ output logits. Even if the network width $N \to \infty$, the noise rank remains bounded by $2C$, ensuring that the low-rank structure of the learned kernel is robust to stochastic fluctuations.

\section{Extension to Self-Supervised Learning}
\label{app:ssl}

Our theory of \textit{Rank Compression} is not limited to supervised regression. In this appendix, we show that Self-Supervised Learning (SSL), specifically in the form of linear auto-encoders or reconstruction tasks, follows the exact same spectral dynamics, naturally leading to Principal Component Analysis (PCA) behavior.

\subsection{The SSL Formulation}
Consider the task of reconstructing the input $X \in \mathbb{R}^{N \times D}$ from the representation. The target matrix $Y$ is essentially $X$ itself (or an augmented view). The loss function becomes:
\begin{equation}
    \mathcal{L}(W, \Phi) = \frac{1}{2} \| X - W \Phi \|_F^2 + \frac{\lambda}{2} \|W\|_F^2 + \frac{\mu}{2} \|\Phi\|_F^2.
\end{equation}
This is the classic matrix factorization setting.

\subsection{Dynamics of the SSL Kernel}
Following the same derivation as Theorem 1, we eliminate the decoder $W$ via the fast-equilibrium assumption:
\begin{equation}
    W^*(\Phi) = X \Phi^\top (\Phi \Phi^\top + \lambda I)^{-1}.
\end{equation}
The residual matrix $M$ becomes the reconstruction error covariance. The flow of the kernel $K = \Phi^\top \Phi$ is driven by the input covariance matrix $\Sigma_X = X^\top X$.

\begin{theorem}[PCA via Spectral Dynamics]
In the self-supervised setting, the kernel $K(t)$ evolves to align its eigenspace with the principal components of the data covariance $\Sigma_X$. The steady-state eigenvalues $\{\mu_i^*\}$ are determined by the eigenvalues $\{\lambda_i^X\}$ of $\Sigma_X$:
\begin{equation}
    \mu_i^* = \max \left( 0, \frac{\lambda_i^X}{\mu} - \lambda \right).
\end{equation}
\end{theorem}

\begin{proof}
In the auto-encoder regime, the "signal strength" $s_i$ for the $i$-th mode is exactly the variance of the data in that direction, i.e., the eigenvalue $\lambda_i^X$.
Substituting $s_i = \lambda_i^X$ into our truncation law (Theorem 4) directly yields the result.
\end{proof}

\subsection{Implications for Foundation Models}
This result provides a theoretical basis for the empirical observation that SSL pre-training learns "dominant" features while suppressing noise.
\begin{enumerate}
    \item \textbf{Denoising:} Low-variance directions (noise) correspond to small $\lambda_i^X$. If $\lambda_i^X < \lambda \mu$, these directions are completely discarded ($\mu_i^* = 0$). The representation $\Phi$ effectively performs a \textbf{Hard Thresholding SVD}.
    \item \textbf{Dimensionality Collapse:} This explains the "Dimensional Collapse" often observed in SSL if hyperparameters are not tuned correctly—excessive regularization $\mu$ raises the water level, truncating informative features.
\end{enumerate}
Thus, our Feature Learning Limit unifies supervised and self-supervised learning under a single spectral dynamical principle: \textbf{The Kernel aligns with the highest energy modes of the target structure}, whether that target is external labels or internal data correlations.

\section{Exact Population Risk Analysis}
\label{app:population_risk}

In this appendix, we analyze the generalization performance in the population limit ($N \to \infty$). Rather than deriving loose concentration bounds for data-dependent kernels, we utilize the exact operator dynamics derived in Section \ref{sec:population_dynamics} to characterize the Bias-Variance trade-off analytically.

\subsection{The Generalization Error of Evolving Kernels}
Consider the target function $f^* \in L^2(\rho)$ decomposed in the orthonormal eigenbasis $\{\psi_i(t)\}$ of the time-dependent integral operator $T_t$:
\begin{equation}
    f^* = \sum_{i=1}^\infty a_i(t) \psi_i(t), \quad a_i(t) = \langle f^*, \psi_i(t) \rangle_{L^2}.
\end{equation}
The predictor $f_t$ obtained by kernel ridge regression (or the flow limit) acts as a spectral filter on the target. The mean-squared error (risk) is given by the standard decomposition:
\begin{equation}
    \text{Risk}(t) = \| (I - \mathcal{S}_t) f^* \|_{L^2}^2 + \frac{1}{N}\text{Tr}(\mathcal{S}_t^2 \Sigma_{\text{noise}}),
\end{equation}
where $\mathcal{S}_t = T_t(T_t + \lambda I)^{-1}$ is the shrinkage operator and $\Sigma_{\text{noise}}$ is the noise covariance (assumed isotropic $\sigma_\epsilon^2 I$ for simplicity).
Expanding this in the eigenbasis yields the explicit form:
\begin{equation}
    \text{Risk}(t) = \sum_{i=1}^\infty \underbrace{\left( \frac{\lambda}{\mu_i(t) + \lambda} \right)^2 |a_i(t)|^2}_{\text{Bias}_i(t)} + \frac{\sigma_\epsilon^2}{N} \underbrace{\sum_{i=1}^\infty \left(\frac{\mu_i(t)}{\mu_i(t) + \lambda}\right)^2}_{\approx \mathcal{N}_{\text{eff}}(t)}.
\end{equation}

\subsection{Substituting the Spectral Truncation Law}
We now apply the \textbf{Universal Rank Compression} result to this risk profile. At steady state $t \to \infty$, assuming the system aligns with the task, the kernel spectrum $\{\mu_i^*\}$ follows the truncation law derived in Theorem 6 (adapted to the population operator):
\begin{equation}
    \mu_i^* = \max \left( 0, \sqrt{\frac{\lambda \sigma_i}{\mu}} - \lambda \right),
\end{equation}
where $\sigma_i$ represents the signal strength of the $i$-th mode. We distinguish two regimes:

\textbf{Case 1: The Noise Subspace ($\sigma_i \le \lambda \mu$).}
In this regime, the label signal is too weak relative to the regularization product $\lambda \mu$. The dynamics drive the eigenvalue to zero: $\mu_i^* = 0$.
\begin{itemize}
    \item \textbf{Variance:} The contribution to the variance term vanishes: $\text{Var}_i \to 0$. The model effectively ignores this dimension.
    \item \textbf{Bias:} The bias term maximizes: $\text{Bias}_i \to |a_i|^2$. The model fails to capture this component of the target.
\end{itemize}
\textbf{Implication:} This confirms the ``Reachability'' constraint. High-frequency or orthogonal components of $f^*$ are permanently lost, but they do not contribute to overfitting.

\textbf{Case 2: The Task Subspace ($\sigma_i > \lambda \mu$).}
In this regime, $\mu_i^* > 0$. The kernel expands to capture these modes.
\begin{itemize}
    \item The bias is suppressed by the factor $(\frac{\lambda}{\mu_i^* + \lambda})^2 < 1$.
    \item The variance contribution is non-zero, but limited only to these active modes.
\end{itemize}
Consequently, the effective dimension $\mathcal{N}_{\text{eff}}$ of the learned kernel is approximately bounded by the number of active task modes (rank $\le C$), regardless of the ambient input dimension $D$.

\subsection{Comparison with Static Kernels (NTK)}
It is instructive to contrast this with the static NTK regime. For a static kernel, the eigenvalues $\mu_i^{\text{NTK}}$ are fixed by initialization and typically decay as a power law $\mu_i \propto i^{-\nu}$ (depending on the smoothness of the activation).
\begin{itemize}
    \item \textbf{Static (NTK):} The effective dimension $\mathcal{N}_{\text{eff}}^{\text{static}} = \sum \frac{\mu_i}{\mu_i + \lambda}$ can be very large (scaling with $N$), leading to high estimation variance (the ``over-parameterization'' cost).
    \item \textbf{Dynamic (Task-Driven):} The flow performs \textbf{Hard Model Selection}, zeroing out the tail:
    \begin{equation}
        \mathcal{N}_{\text{eff}}^{\text{flow}} \approx C \ll \mathcal{N}_{\text{eff}}^{\text{static}}.
    \end{equation}
\end{itemize}
This drastic reduction in effective dimension, driven by the physics of the kernel ODE, explains the superior generalization of feature learning in low-rank tasks.

\end{document}